\documentclass[letterpaper]{article} 
\usepackage[preprint]{paperstyle}  
\usepackage[hyphens]{url}  
\usepackage{graphicx} 
\usepackage{natbib}  
\usepackage{caption} 
\usepackage{amsmath}
\usepackage{amssymb}
\usepackage{amsthm}
\newtheorem{proposition}{Proposition}
\newtheorem{theorem}{Theorem}
\newtheorem{definition}{Definition}
\usepackage{booktabs}
\usepackage{algorithm}
\usepackage{algorithmic}
\usepackage{placeins}

\newcommand{\cmark}{\checkmark}
\newcommand{\xmark}{$\times$}

\title{Beyond Aggregate Risk:\\
Role-Stratified Conformal Risk Control for LLM Tool Calls}

\author{%
Md Ashikur Rahman,
Md Arifur Rahman,
Niamul Hassan Samin,\\
Khandaker Rifah Tasnia,
Md Hasibul Amin,
Sifat Rahman Ahona,\\
Juena Ahmed Noshin
}
\affiliations{%
}

\begin{document}
\maketitle

\begin{abstract}
Language-model agents act through structured tool calls whose arguments carry
very different risks: untrusted content may legitimately shape an email body but should
never set a recipient, account, command, or credential. Existing conformal risk
control methods certify a tool call as a whole, so a failure in one rare
high-risk field can be averaged away by the many benign arguments around it,
leaving the argument that causes harm uncertified. We introduce role-stratified per-field conformal risk
control, a calibration layer that wraps any per-field detector and assigns a
separate threshold and risk budget to each semantic argument role. We show that
aggregate certification pays a price of coarseness, tightening a rare role's
effective budget in proportion to how often that role appears, whereas
role-stratified calibration certifies each sufficiently sampled role directly
with a finite-sample guarantee and pools the rarest roles. Across AgentDojo and
InjecAgent with six language models, our method achieves the most consistent
role-specific budget compliance among the methods we evaluate under model and
attack transfer, detector noise, gradual drift, unseen tool suites, and adaptive
attacks, providing formal per-role guarantees under exchangeability or after
recalibration. These results
suggest that structured tool calls should be certified at the semantic-role
level, not the whole action.
\end{abstract}

\section{Introduction}

Consider an agent asked to pay an invoice. It reads the amount from an email,
but hidden text redirects the payment to a new account. Payment calls
specify an account and amount, just as a \texttt{send\_email} call specifies a
recipient, body, and attachments. Untrusted content may legitimately influence an email
body but should not determine a recipient, payment account, or credential. When
text from a webpage, email, or retrieved document changes such a high-risk field,
the agent may act against the user's intent~\citep{greshake2023notwhat}.

Existing defenses follow two main approaches. Security-enforcement systems track
high-risk fields and apply hard allow-or-deny
rules~\citep{camel2025,fides2025,pact2026}. These systems can provide strong
semantic guarantees, including noninterference, but do not offer tunable risk
budgets or statistical bounds on residual violations. Statistical calibration
methods, including conformal risk control (CRC), provide tunable finite-sample
guarantees but usually control the action as a
whole~\citep{angelopoulos2024conformal,cora2026}.

Action-level control can hide failures in rare high-risk fields: averaging
across arguments may keep aggregate loss below budget even when a recipient,
account, or credential field exceeds its limit, so the unit being certified
differs from the unit that can cause harm.

We address this mismatch with role-stratified per-field conformal risk control.
Each argument is assigned a semantic role, such as \texttt{target},
\texttt{credential}, \texttt{command}, or \texttt{content}, with its own
threshold and risk budget. This applies class-conditional, or Mondrian,
conformal control at the role
level~\citep{vovk2003mondrian,ding2023classconditional}. Stratification
prevents aggregate dilution, while conformal calibration provides finite-sample
validity and supports recalibration under changing conditions. The method can
wrap any per-field detector and improves directly as the detector improves.

Our central thesis is that the appropriate calibration stratum is the semantic
role of each argument. Our contributions organize around that claim.
\begin{itemize}
\item \emph{Why role-level control is necessary.}
  Aggregate certification dilutes rare-role failures: a controller that observes
  only aggregate loss inflates a role-$r$ violation rate by $1/p_r$
  (Proposition~\ref{prop:gap}) and can certify the role only by shrinking its
  effective budget to $\alpha\,p_r$ (Proposition~\ref{prop:lb}).
  Theorem~\ref{thm:frontier} orders observation channels by this price of
  coarseness, specializing Blackwell's comparison of
  experiments~\citep{blackwell1953equivalent} to conformal risk certification.
\item \emph{How to obtain it.}
  Role-stratified per-field CRC assigns each role its own budget, a
  certifiability floor of $1/(n_r{+}1)$, and a simultaneous high-probability
  certificate over the final calibration strata, with pooling for rare roles
  (Theorem~\ref{thm:simul}).
  Label-conditional calibration further yields a prevalence-invariant
  attack-conditional guarantee on
  $V_{\mathrm{att}}(r)$~\citep{podkopaev2021labelshift}; in the zero-budget
  limit the method recovers detector-relative noninterference
  (Proposition~\ref{prop:ifc}).
\item \emph{How it behaves in practice.}
  Across six models and two benchmarks, the empirical utility gap follows the
  predicted price of coarseness. Role-stratified per-field CRC remains
  compliant under transfer, unseen suites, detector noise, gradual drift, and
  adaptive attacks. A $2{\times}2$ ablation shows stratification provides the
  main robustness gain, while conformal recalibration restores finite-sample
  validity under the new condition.
\end{itemize}

We use \emph{certificate} only under exchangeability or after recalibration.
Under distribution shift with frozen thresholds, we report empirical budget
compliance rather
than a conformal certificate.

\begin{figure*}[t]
  \centering
  \includegraphics[width=\textwidth]{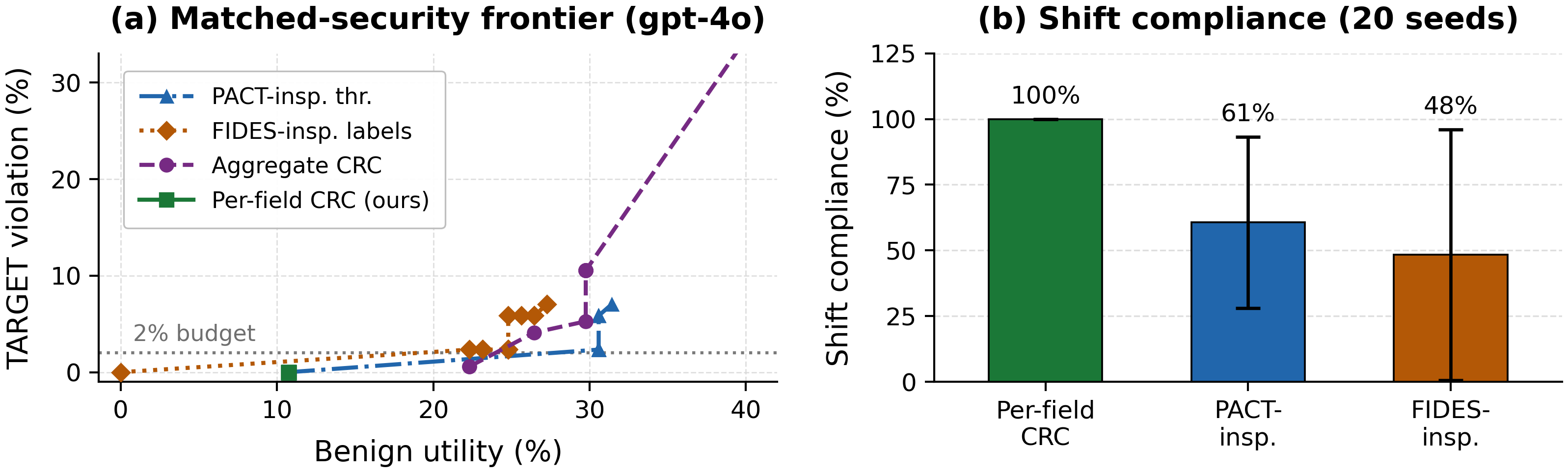}
  \caption{(a)~Utility versus \texttt{target} violation on
    GPT-4o at a $2\%$ budget using the deployable score over 20 seeds.
    (b)~Empirical compliance across eight shifted conditions:
    per-field CRC $100\%\pm0\%$, PACT-inspired threshold $61\%\pm33\%$, and
    FIDES-inspired labels $48\%\pm48\%$.
    Results in panel (b) report empirical compliance under shift.}
  \label{fig:headline}
\end{figure*}

\section{Related Work}

\paragraph{Security enforcement for tool-using agents.}
One line of work protects tool calls through deterministic information-flow
control (IFC). CaMeL isolates high-risk calls using a dual-LLM, capability-based
design~\citep{camel2025}. FIDES tracks integrity and confidentiality labels for
individual values~\citep{fides2025}, while PACT traces argument-level provenance
using the \texttt{target}, \texttt{command}, \texttt{credential}, and
\texttt{content} roles~\citep{pact2026}, which we adopt and extend with
\texttt{selector} and \texttt{control} for AgentDojo. ARM and AgentVisor study
causal provenance and semantic
isolation~\citep{arm2026,agentvisor2026}.

These systems build on classical information-flow control and
noninterference~\citep{denning1976lattice,volpano1996sound}. They provide strong
semantic guarantees, but generally rely on fixed allow-or-deny decisions rather
than tunable risk budgets or statistical bounds on residual violations.

\paragraph{Statistical calibration and conformal risk control.}
A second line of work provides tunable, distribution-free guarantees through
calibration. CORA applies conformal risk control to mobile GUI agents but
defines risk over the action as a whole~\citep{cora2026}. Other methods control
hallucination or response-level error through abstention or calibrated
refusal~\citep{conformalabstention2024,conformalactuator2025}. These
approaches also operate at the response or action level rather than the
semantic-role level.

Our method builds on conformal risk
control~\citep{angelopoulos2024conformal}, split conformal
prediction~\citep{papadopoulos2002inductive,vovk2005algorithmic,lei2018distribution,angelopoulos2023gentle},
Learn-then-Test~\citep{angelopoulos2025ltt}, and risk-controlling prediction
sets~\citep{bates2021rcps}. Its structure follows class-conditional,
group-conditional, and Mondrian conformal
methods~\citep{vovk2003mondrian,ding2023classconditional,bairaktari2025kandinsky}.
The key difference is the unit of control: we calibrate risk separately for each
semantic argument role.

\paragraph{Concurrent conformal defenses and calibration units.}
Concurrent methods calibrate entire trajectories or streams of agent
runs~\citep{toolchaincrc2026,hultberg2026anytime,khosravi2026selective},
whereas we certify semantic argument roles. These temporal and semantic
calibration units are complementary.

\paragraph{Prompt injection, benchmarks, and positioning.}
These research directions meet in indirect prompt injection, a practical threat
to LLM-integrated systems~\citep{greshake2023notwhat}. We evaluate on AgentDojo
and InjecAgent, two benchmarks for prompt-injection attacks against tool-using
agents~\citep{debenedetti2024agentdojo,zhan2024injecagent}.

Training- and prompt-level defenses, including StruQ, SecAlign, Spotlighting,
the Instruction Hierarchy, and Jatmo, aim to prevent or weaken attacks before
calibration and are therefore
complementary to our
method~\citep{struq2025,secalign2025,spotlighting2024,instructionhierarchy2024,jatmo2024}.
Broader studies classify and compare prompt-injection attacks and
defenses~\citep{ipisok2025,liu2024formalizing}.

Our contribution is a calibration layer that wraps any per-field detector with
role-specific risk budgets, orthogonal to existing defenses rather than a new
attack or detector. We compare controlled PACT-, FIDES-, and CaMeL-inspired
proxies under a common detector to isolate enforcement granularity; these are
not full reproductions of the original systems.

\section{Problem Setup and Threat Model}

\paragraph{Structured actions and per-field violations.}
A tool call is a structured action
\begin{equation}
A = (\textit{op},\, x_1, \dots, x_k),
\end{equation}
where $\textit{op}$ is an operation, such as \texttt{send\_email}, and each
$x_i$ is a named argument. Every argument is assigned a semantic role
$r(i) \in \mathcal{R}$, such as \texttt{target}, \texttt{credential},
\texttt{command}, \texttt{selector}, \texttt{control}, or \texttt{content}.

These roles carry different risks. A \texttt{target} or \texttt{credential}
determines where an action goes or which authority it uses, whereas
\texttt{content} is often expected to reflect untrusted input. Fields are the
enforcement and sampling units. Semantic roles are the calibration and
certification strata. A field is violated when untrusted input changes its
value against the user's intent. It is allowed when the controller permits that
value to be emitted. For each role $r$, we define
\begin{equation}
\label{eq:vdef}
V(r) = \Pr\!\big[\,\text{violated}\wedge\text{allowed}\ \big|\ \text{role}=r\,\big],
\end{equation}
the probability that an emitted role-$r$ field is both violated and allowed.
We certify violation rates over role-specific field populations, not the safety
of an individual tool call. We require $V(r) \le \alpha(r)$, with tighter
budgets for higher-risk roles.

\paragraph{What $V(r)$ measures.}
The denominator of $V(r)$ includes all emitted role-$r$ fields, both benign and
attacked, so it is the quantity optimized by the calibration loss and reported
throughout. Because clean fields are included, $V(r)$ depends on operational
attack prevalence and is best read as field-stream integrity under a stated
clean-attack mixture, not a prevalence-invariant leakage rate. To isolate
attack-conditional behavior, we also report
\begin{equation}
\label{eq:vattdef}
V_{\mathrm{att}}(r) = \Pr\!\big[\,\text{violated}\wedge\text{allowed}\ \big|\
  \text{role}=r,\ \text{attacked}\,\big],
\end{equation}
which is prevalence-independent. By construction, an unattacked field cannot
receive an attack-induced violation label, so its clean-conditional rate is
zero and $V(r) = \pi(r)\,V_{\mathrm{att}}(r)$ for field-level attack prevalence
$\pi(r)$.
In our benchmarks the \texttt{target} stream is attack-heavy
($\pi \approx 85\%$-$94\%$), so the two are similar. They diverge in a
low-prevalence deployment (Supplement~A). Both differ from
trace-level attack success rate (ASR), the fraction of attacked episodes in
which any high-risk field is leaked, which we report as a separate security
metric, not the certified object.

\paragraph{Attacker.}
The attacker controls untrusted text from webpages, emails, or retrieved
documents and attempts to influence one or more high-risk arguments. In the
strongest setting, the attacker selects the least detectable successful prompt
from several candidates and uses attack families unseen during calibration. The
attacker does not control the calibration data or model weights. We also exclude
gradient-based prompt optimization~\citep{zou2023universal}.

\paragraph{Per-argument score.}
Each argument $x_i$ receives a nonconformity score $s_i$ that estimates its
influence from untrusted input. Our deployable score is annotation-free:
\begin{equation}
\begin{aligned}
s_i &= \mathrm{sim}(x_i, \text{untrusted context})\\
    &\quad - \mathrm{sim}(x_i, \text{trusted user prompt}).
\end{aligned}
\end{equation}
We shift and rescale the similarity difference to $[0,1]$. Larger scores
indicate stronger influence from untrusted context, so a field is allowed when
$s_i\le\tau$. The score does not use attacker labels. For analysis, we also
evaluate an attacker-literal overlap diagnostic that measures lexical overlap
between the emitted field and the injection goal. Because it reads the attacker
literal used to define the violation label, its ROC-AUC of $0.93$ partly reflects
access to evaluation-specific attacker text, so we treat it as a label-coupled
oracle-aided upper bound rather than the deployable default. This diagnostic is
distinct from the oracle-aided counterfactual clean-vs-injected overlap
($0.76$ ROC-AUC) evaluated only in the non-verbatim stress test (Supplement~A).
Both scores use the same calibration and decision rule. Only the score
computation changes, as detailed in Supplement~A.

\section{The Aggregate-Budget Failure}

Aggregate control bounds expected harm by averaging risk across all fields in a
tool call. This is misleading when structured calls mix fields with very
different consequences: if violations concentrate in a rare high-risk role, many
benign fields keep the average below budget even when that role exceeds its
limit, and raising the aggregate budget only weakens the overall constraint
without protecting the risky field.

\begin{proposition}[Concentration gap]
\label{prop:gap}
Let $p_r > 0$ be the fraction of fields with role $r$, where
$\sum_r p_r = 1$. Define aggregate field-averaged harm as
$\overline{V} = \sum_r p_r V(r)$. If aggregate control guarantees
$\overline{V} \le \alpha_{\mathrm{agg}}$, then
$V(r) \le \alpha_{\mathrm{agg}}/p_r$ for every role $r$, and this bound is
tight.
\end{proposition}

\begin{proof}
Because all terms are non-negative,
$p_r V(r) \le \overline{V} \le \alpha_{\mathrm{agg}}$,
so $V(r) \le \alpha_{\mathrm{agg}}/p_r$. Equality is possible when all
violations occur in role $r$.
\end{proof}

Thus, aggregate protection weakens in inverse proportion to role prevalence.
In the full clean-attack evaluation population used to compute $\overline{V}$,
\texttt{target} fields have prevalence $p_{\texttt{target}}\approx 0.12$ (pooled
across the six models, $0.06$-$0.21$ per model), consistent with the denominator
of $V(r)$ in Equation~\eqref{eq:vdef}, which likewise includes all role-$r$
fields. Thus, even before distribution shift, a $1\%$ aggregate budget can
permit roughly $1\%/0.12 \approx 8\%$ \texttt{target} violations.

The next result shows that aggregate-only certification must reduce its effective
budget by the same factor $p_r$. Two distributions can have identical scores and
aggregate losses while placing every violation in different roles. An
aggregate-only controller cannot distinguish them.

\begin{definition}[Aggregate-measurable controller]
\label{def:agg-measurable}
Let $\Phi(A)$ denote the score information visible to the controller for action
$A$, such as action-level or field-level detector scores. Define the action's
average violation loss as
\begin{equation*}
L_{\mathrm{agg}}(A)
  = \frac{1}{|A|}\sum_{i\in A}
    \mathbf{1}\{\text{field $i$ is violated and allowed}\}.
\end{equation*}
A policy is \emph{aggregate-measurable} if it uses calibration data only through
$\{(\Phi(A_j),L_{\mathrm{agg}}(A_j))\}_j$ and makes decisions using only
$\Phi(A)$. Aggregate control and single-threshold defenses belong to this class.
Role-stratified CRC does not, because it also observes the role of the violated
field.
\end{definition}

\begin{proposition}[Aggregate control cannot certify a rare role]
\label{prop:lb}
Fix a role $r$ whose field-level prevalence $p_r=a/K\in(0,1)$ is realizable in a
finite benchmark (integers $1\le a\le K-1$), and a budget $\alpha\in(0,1)$.
There exist
two distributions, $P_0$ and $P_1$, with the same role prevalence $p_r$ and
identical observables $(\Phi,L_{\mathrm{agg}})$. Any aggregate-measurable
controller $\pi$ therefore has the same aggregate violation
$\overline{V}(\pi)$ and benign utility under both distributions. However,
$V(r)=0$ under $P_0$, while $V(r)=\overline{V}(\pi)/p_r$ under $P_1$.

To guarantee $V(r)\le\alpha$ under both distributions, the controller must
therefore enforce $\overline{V}(\pi)\le\alpha\,p_r$. This shrinks the effective
aggregate budget by a factor of $p_r$ and lowers utility even under $P_0$, where
role $r$ is never violated. Role-stratified calibration instead guarantees
$V(r)\le\alpha$ whenever $\alpha\ge 1/(n_r{+}1)$.
\end{proposition}

\noindent\emph{Proof sketch.} Write the field-level prevalence as $p_r=a/K$ with
integers $1\le a\le K-1$; every prevalence realizable in a finite benchmark is
rational and arises this way, and the special case $a=1$ recovers $K=1/p_r$. Give
each action $K$ fields, $a$ of which have role $r$. Mark some actions as bad by
violating exactly one field while preserving the same score distribution for bad
and clean actions. Under $P_1$, place the violation on a role-$r$ field; under
$P_0$, place it on a \texttt{content} field. The controller therefore behaves
identically under both distributions, yet under $P_1$ all harm concentrates on
role $r$, giving $V(r)=\overline{V}(\pi)/p_r$. The full proof appears in
Supplement~D.

Propositions~\ref{prop:gap} and~\ref{prop:lb} are endpoints of a broader
ordering. We order controllers by observation channel, writing
$\Phi_1\preceq\Phi_2$ when $\Phi_2$ refines $\Phi_1$. The hierarchy consists of
the aggregate channel $\Phi_{\mathrm{agg}}$, the role channel
$\Phi_{\mathrm{role}}$, which adds each field's role, and the field channel
$\Phi_{\mathrm{field}}$.

\begin{theorem}[Certification-granularity frontier]
\label{thm:frontier}
Let $u^\star(\beta)$ be the largest benign utility of any
$\Phi_{\mathrm{agg}}$-measurable controller satisfying aggregate violation
$\overline{V}\le\beta$. Then:
\emph{(i)~Achievability.} On $\Phi_{\mathrm{role}}$,
Equation~\eqref{eq:threshold} certifies every budget vector satisfying
$\alpha(r)\ge 1/(n_r{+}1)$.
\emph{(ii)~Price of coarseness.} Any $\Phi_{\mathrm{agg}}$-measurable
controller that certifies $V(r)\le\alpha$ must enforce
$\overline{V}\le\alpha\,p_r$, so its utility is at most $u^\star(\alpha p_r)$.
\emph{(iii)~Monotonicity.} At fixed utility, the set of certifiable budget
vectors is nondecreasing along
$\Phi_{\mathrm{agg}}\preceq\Phi_{\mathrm{role}}\preceq\Phi_{\mathrm{field}}$.
\end{theorem}

Part~(iii) follows Blackwell's comparison of
experiments~\citep{blackwell1953equivalent}. Our contribution is the
closed-form price of coarseness in part~(ii) and the constructive conformal
method achieving the bound in part~(i). The full proof appears in Supplement~D.

The empirical results reflect this gap. On AgentDojo with GPT-4o, a $10\%$
aggregate budget allows $10.5\%$ \texttt{target} violations, whereas per-field
CRC allows $0\%$ (Supplement~B, Fig.~B.1 and Table~B.1). On InjecAgent, under a
$10\%$ aggregate budget, aggregate control reaches only $2.9\%$ overall
violation while allowing $100\%$ of attacked \texttt{target} fields
(Supplement~B).

Reweighting does not remove the problem because any single-threshold action loss
can still dilute role-specific failures.

\begin{table}[t]
\centering
\small
\setlength{\tabcolsep}{4pt}
\begin{tabular}{@{}lcc@{}}
\toprule
Action loss & Viol.\ \% & Util.\ \% \\
\midrule
Mean-aggregate CRC          & 1.1-5.0  & 23.6-34.3 \\
Risk-weighted               & 0.6-3.2  & 14.7-32.6 \\
Inv.-prevalence weighted    & 0.0-3.2  & 0.4-32.6 \\
Max-risk (any high-risk)    & 1.1-21.2 & 23.3-43.0 \\
Per-field CRC (ours)        & 0.0-0.3  & 9.4-14.9 \\
\bottomrule
\end{tabular}
\caption{\texttt{Target} violation and abstain utility for single-threshold
  baselines and per-field CRC at a $1\%$ \texttt{target} budget across six
  models, using the label-coupled diagnostic score. Ranges are across-model
  mins and maxes of the $20$-split mean violation in this channel-ablation
  setting; the single-split per-model panel in Supplement~B, Table~B.1 is a
  separate protocol, not a numerical expansion of this table.}
\label{tab:strongagg-main}
\end{table}

Table~\ref{tab:strongagg-main} compares mean, risk-weighted, inverse-prevalence,
and max-risk single-threshold losses calibrated to a $1\%$ \texttt{target}
budget. Each either exceeds budget, reaching up to $21.2\%$ \texttt{target}
violation, or drops utility to $0.4\%$ on one model. Among these methods,
only per-field CRC remains within budget, with $0.0\%$-$0.3\%$ violation and
$9.4\%$-$14.9\%$ utility. Protecting a rare high-risk role therefore requires
certifying that role directly, motivating the method developed next.

\section{Method: Role-Stratified Per-Field CRC}

We apply conformal risk control separately to each semantic role, assigning each
role its own threshold and risk budget.

\paragraph{Role-specific calibration.}
For each role $r$, the calibration set contains clean and injected episodes with
field-level violation labels. For a field $x$, define
\begin{equation*}
L_r(x;\tau) = \mathbf{1}\{s(x)\le\tau \ \wedge\ \text{$x$ is violated}\},
\end{equation*}
where $s(x)$ is its nonconformity score and $\tau$ is the allow threshold. The
loss is monotone in $\tau$: increasing the threshold allows more fields and can
only add violations.

Let $n_r$ be the number of calibration fields with role $r$. Following conformal
risk control~\citep{angelopoulos2024conformal}, we choose the largest threshold
whose inflated empirical risk remains within the role-specific budget
$\alpha(r)$:
\begin{equation}
\label{eq:threshold}
\hat{\tau}(r)
  = \sup\Big\{\tau :
    \tfrac{1}{n_r+1}
    \Big(\textstyle\sum_{i=1}^{n_r} L_r(x_i;\tau) + 1\Big)
    \le \alpha(r)\Big\}.
\end{equation}
Under exchangeability between calibration and test fields of the same role,
$\mathbb{E}[L_r(X_{\mathrm{new}};\hat{\tau}(r))] \le \alpha(r)$.
The conformal result is standard. The key design choice is to use the semantic
role, rather than the full action, as the unit of control.

The exchangeable unit is the individual field. Our implemented calibration
includes every emitted role-$r$ field. If an episode contains several
same-role fields they may be correlated, so $n_r$ can overstate the effective
sample size; a conservative one-per-episode alternative is discussed in
Supplement~A (code archive) and is not used for the main results.

\paragraph{Runtime enforcement.}
Runtime enforcement operates independently at the field level, preserving
unaffected arguments rather than rejecting the entire call. Let $g(r)$ map each
role to its final stratum: $g(r)=r$ for an individually certifiable role, and
$g(r)$ is the pooled high-risk stratum otherwise. A field of role $r_j$ is then
compared with the threshold $\hat{\tau}(g(r_j))$ of its final stratum.

\begin{algorithm}[t]
\caption{Role-Stratified Per-Field CRC}
\label{alg:pfcrc}
\begin{algorithmic}[1]
\REQUIRE Calibration fields $\{(x_i, r_i, y_i)\}$ with scores $s_i$,
  violation labels $y_i$, and per-role budgets $\alpha(r)$.
\STATE \textbf{Calibration:}
\FOR{each role $r \in \mathcal{R}$}
  \STATE $n_r \gets |\{i : r_i = r\}|$
  \IF{$1/(n_r{+}1) > \alpha(r)$}
    \STATE $g(r) \gets$ pooled high-risk stratum ($r$ is not individually
      certifiable)
  \ELSE
    \STATE $g(r) \gets r$
  \ENDIF
\ENDFOR
\STATE Form the final calibration set for each stratum in $g(\mathcal{R})$
\STATE Compute one threshold $\hat{\tau}(g)$ per final stratum $g$ using
  Equation~\eqref{eq:threshold}
\STATE \textbf{Enforcement:}
\FOR{each argument $x_j$ with role $r_j$ and score $s_j$}
  \IF{$s_j \le \hat{\tau}(g(r_j))$}
    \STATE Allow $x_j$
  \ELSE
    \STATE Revert $x_j$ to its trusted value when available, otherwise abstain
      or escalate
  \ENDIF
\ENDFOR
\end{algorithmic}
\end{algorithm}

\paragraph{Certifiability and rare roles.}
A role can be certified only when its calibration set is sufficiently large.
From Equation~\eqref{eq:threshold}, the smallest certifiable budget is
$1/(n_r{+}1)$. If $\alpha(r)$ lies below this floor, the role cannot be
certified separately.

For \texttt{target}, the calibration sets contain $145$-$531$ fields per model,
giving floors of $0.2\%$-$0.7\%$. This supports a $1\%$ budget, with empirical
violations of $0.0\%$-$0.3\%$ across all six models, as shown in
Table~\ref{tab:strongagg-main}.

The \texttt{credential} role is much rarer, with only $0$-$22$ calibration
fields per split across the six models. Computing the finite-sample floor
$1/(n_r{+}1)$ on each split and averaging over the $20$ splits gives per-model
mean floors of $6.7\%$ (GPT-4o) to $63.6\%$ (Qwen2.5-7B), all far above $1\%$
(Supplement~C, Table~C.1). This range is a span of per-model averages, not a
span of individual splits. The \texttt{credential} role therefore cannot be
certified individually at a $1\%$ budget, so we pool it with the other high-risk
roles $\{\texttt{target},\texttt{credential},\texttt{command}\}$.
The pooled group has a floor of $0.14\%$-$0.49\%$, making a $1\%$ certificate
possible for every model, with $0\%$ empirical \texttt{credential} violation.

The certificate applies to the pooled group, not to \texttt{credential} alone:
by Proposition~\ref{prop:gap} a pooled $1\%$ budget implies only the individual
bound $1\%/p_{\mathrm{cred}\mid\mathrm{pool}}$, so we report \texttt{credential}
as empirically controlled rather than individually certified. Even pooled across
models only about $91$ \texttt{credential} fields exist, versus the $\approx300$
a $1\%$, $\delta=0.05$ certificate would need. The data supports about a $3\%$
certificate. This is a data, not estimator, limitation (Supplement~C, Table~C.1).

\paragraph{Relation to information-flow control.}
Assign each role a label in an information-flow lattice. The integrity-critical
roles $\{\texttt{target},\texttt{credential},\texttt{command}\}$
are sinks that untrusted input must not reach, whereas \texttt{content} is
declassifiable.

\begin{proposition}[Conformal relaxation of noninterference]
\label{prop:ifc}
Under this lattice, as $\alpha(r)\to 0$ for every integrity-critical role, the
rule in Equation~\eqref{eq:threshold} converges to a deterministic monitor that
blocks every field flagged by the detector. In the zero-budget policy limit,
the method therefore recovers detector-relative noninterference. This is a
conceptual limit. Finite-sample certification remains subject to the floor
$1/(n_r{+}1)$. For $\alpha(r)>0$, the method is a tunable relaxation that
permits a certified residual-violation budget in exchange for utility.
\end{proposition}

Unlike quantitative information flow, which bounds leakage
measures~\citep{smith2009qif}, this method provides a finite-sample,
distribution-free bound on residual violations and can inherit any per-field
detector.

\paragraph{Simultaneous high-probability certification.}
Equation~\eqref{eq:threshold} controls each role in expectation. Deployment may
instead require all final calibration strata to satisfy their budgets
simultaneously. Let $\mathcal{R}_c=\{r:\alpha(r)<1\}$ and let
$g:\mathcal{R}_c\to\mathcal{G}$ be the stratum map from Algorithm~\ref{alg:pfcrc}
($g(r)=r$ when $r$ is individually certifiable; otherwise $g(r)$ is the pooled
high-risk stratum), with $\mathcal{G}=g(\mathcal{R}_c)$.

Because the per-stratum loss is Bernoulli, and treating the stratum-$g$
calibration fields as independent draws, let
$U(k_g(\tau), n_g; \delta_g)$
be the exact Clopper-Pearson upper bound, where $k_g(\tau)$ counts calibration
fields of stratum $g$ that are both allowed and violated. This exactness requires
field-level independence within each stratum; when episodes emit several
correlated same-role fields, the episode-level construction of Supplement~A
restores it. We choose
\begin{equation}
\label{eq:cpthresh}
\hat{\tau}_g(\delta_g)
  = \sup\bigl\{\tau : U(k_g(\tau), n_g; \delta_g) \le \alpha(g)\bigr\}.
\end{equation}

\begin{theorem}[Simultaneous stratum certificate]
\label{thm:simul}
Let $\delta=\sum_{g\in\mathcal{G}}\delta_g$. Then, with probability at least
$1-\delta$ over the calibration sample, $V(g)\le\alpha(g)$ holds simultaneously
for every $g\in\mathcal{G}$. Individually retained roles therefore receive
role-specific certificates; pooled rare roles receive only a pooled-stratum
certificate.
\end{theorem}

The proof handles discrete scores and ties by reducing the failure event, via
monotonicity, to a Clopper-Pearson underestimate at a single deterministic
population boundary $\tau_g^{\star}=\inf\{\tau:V(g)(\tau)>\alpha(g)\}$ rather than
at a sample-dependent candidate, then applying a union bound over the final
strata. The full proof appears in Supplement~D.

\paragraph{Prevalence-invariant certification.}
Conditioning calibration on attacked role-$r$ fields yields, through
label-conditional conformal calibration~\citep{podkopaev2021labelshift}, a
certificate on the attack-conditional rate $V_{\mathrm{att}}(r)$ that is
invariant to test-time attack prevalence $\pi(r)$, whereas the mixture rate
$V(r)=\pi(r)\,V_{\mathrm{att}}(r)$ decreases as attacks become rarer
(Supplement~A; Table~A.1 reports the attack-heavy $V\approx V_{\mathrm{att}}$ case).

\paragraph{Independent confidence allocation.}
The confidence allocation $\{\delta_g\}$ is selected using an allocation set
$D_{\mathrm{alloc}}$ that is independent of the certification set
$D_{\mathrm{cert}}$. Final stratum thresholds and Clopper-Pearson bounds
are computed only from $D_{\mathrm{cert}}$. Conditional on $D_{\mathrm{alloc}}$,
the selected allocation is fixed, so Theorem~\ref{thm:simul} applies to
$D_{\mathrm{cert}}$.

We divide $\delta$ uniformly across final strata using the Bonferroni allocation
$\delta_g = \delta/|\mathcal{G}|$.
A data-dependent water-filling allocation that equalizes marginal utility across
strata is also valid under concavity. However, it provides no measurable gain over
the uniform split under finite-sample estimates (Supplement~E), so we use
uniform allocation as the deployable default.

\paragraph{Behavior under distribution shift.}
A fixed per-role threshold can perform well under stable conditions, but
Equation~\eqref{eq:threshold} assumes exchangeability between calibration and
test fields of the same role, which may fail after transfer to a new model,
attack family, or environment. Conformal calibration adds finite-sample validity
through Theorem~\ref{thm:simul} and lets thresholds be recalibrated when the
score distribution changes, restoring compliance with a budget that a frozen
threshold may exceed.

\begin{proposition}[Per-role degradation under shift]
\label{prop:shift}
Fix a role $r$ and threshold $\hat{\tau}(r)$, calibrated on source distribution
$P_r$, such that
$\mathbb{E}_{P_r}[L_r(X;\hat{\tau}(r))]\le\alpha(r)$ and $L_r\in[0,1]$.
For any test distribution $Q_r$,
$\mathbb{E}_{Q_r}[L_r(X;\hat{\tau}(r))]
  \le \alpha(r) + \mathrm{TV}(P_r,Q_r)$.
Recalibrating on labeled samples from $Q_r$ removes the shift term and restores
a fresh certificate under exchangeability with $Q_r$.
\end{proposition}

This is the standard bounded-function expectation-difference
bound~\citep{barber2023beyond}, not a distribution-free certificate for shifted
data. It is useful only when $\mathrm{TV}(P_r,Q_r)$ is known or tightly bounded,
and motivates recalibration. Together with Proposition~\ref{prop:gap}, it shows
that aggregate control suffers a $1/p_r$ inflation even before shift, whereas
per-role control degrades only with the removable role-specific shift
$\mathrm{TV}(P_r,Q_r)$.

\section{Experiments}

\paragraph{Setup.}
We evaluate on AgentDojo~\citep{debenedetti2024agentdojo}, which covers banking,
workspace, Slack, and travel, and on InjecAgent~\citep{zhan2024injecagent}. The
six models are GPT-4o and GPT-4o-mini~\citep{openai2024gpt4o}, Gemini~2.5 Pro
and Flash~\citep{comanici2025gemini25}, Llama~3.3~70B~\citep{grattafiori2024llama3},
and Qwen2.5-7B~\citep{qwen2024qwen25}.

We run a \emph{controlled trace-replay evaluation} on real agent traces from six
frontier models under live prompt-injection attacks: scores are replayed from
fixed trajectories, and calibration/test splits are resampled without additional
model calls. Trace-replay isolates field-level calibration behavior and makes
every number reproducible from the archive.
Role budgets are $\alpha(r)=1\%$ (in-distribution) or $2\%$ (shift) for
\texttt{target}/\texttt{credential}. $\alpha=0.10$ is the aggregate budget for
whole-action baselines only. We measure utility under conservative abstention
and report over-intervention with every safety result. Split sizes,
hyperparameters, seeds, and reproduction scripts are provided in the code
archive; Supplement~F summarizes reproducibility and infrastructure.
Notation is summarized in Supplement~G.

\paragraph{Safety and utility.}
Abstain utility is trace-replay task success when flagged calls are rejected; revert
utility restores flagged fields to trusted values. Because neither intervention
re-runs the agent, both estimate utility on frozen trajectories.
Over-intervention is the fraction of benign fields modified or blocked, and
shift compliance is the fraction of shifted conditions meeting the target
budget. We compare these quantities with \texttt{target} violation at matched
operating points (Figure~\ref{fig:headline}a). Under shift, per-field CRC reaches
$0\%$ ASR at $30.6\%$ over-intervention. The trade-off is in-distribution
utility: there the PACT-inspired threshold attains higher utility at similar
\texttt{target} safety, with $27.2\%$/$32.3\%$ abstain/revert utility, compared
with $9.9\%$/$25.2\%$ for per-field CRC. Table~\ref{tab:capability} reports the
complete safety-utility comparison; per-model aggregate-failure details appear
in Supplement~B, Table~B.1.

\paragraph{Safety with an annotation-free detector.}
The detector score uses no attacker-text annotations, though conformal
calibration still uses field-level violation outcomes. At a $2\%$ \texttt{target}
budget, both the label-coupled diagnostic score and the annotation-free
provenance score keep worst-case \texttt{target} violation $\le0.3\%$ with at
least $95\%$ held-out compliance across all six models, despite the deployable
score's much lower ROC-AUC ($0.71$ versus $0.93$) (Supplement~C, Table~C.2).
Calibration thus determines validity, while detector quality primarily
determines utility.

\begin{table*}[t]
\centering
\small
\setlength{\tabcolsep}{4pt}
\begin{tabular}{@{}lccccccc@{}}
\toprule
Defense & Tun.\ & Cert.\ & Shift \% & ASR & Abs.\ & Rev.\ & Over \\
\midrule
No defense & \xmark & \xmark & - & 39.7 & 35.1 & 35.1 & 0.0 \\
CaMeL-insp.\ (whole-action) & \xmark & \xmark & - & 1.0 & 22.0 & 27.2 & 44.8 \\
FIDES-insp.\ (IFC) & \xmark & \xmark & $48{\pm}48$ & 2.1 & 22.2 & 27.5 & 30.4 \\
PACT-insp.\ (per-arg.) & \xmark & \xmark & $61{\pm}33$ & 2.1 & 27.2 & 32.3 & 6.6 \\
Aggregate CRC & \cmark & agg.\ & - & 6.0 & 27.4 & 35.0 & 16.4 \\
Per-field CRC (ours) & \cmark & role & $100{\pm}0$ & 0.0 & 9.9 & 25.2 & 30.6 \\
\bottomrule
\end{tabular}
\caption{Six-model trace-replay comparison (deployable provenance score).
  Tun./Cert.\ = tunable budget / certificate level (agg.\ or per-role);
  Shift \% = \texttt{target} compliance on eight frozen transfers;
  Abs./Rev./Over = abstain utility / revert utility / over-intervention.
  Security rows are controlled proxies, not full PACT/FIDES/CaMeL
  reproductions. Utility is estimated on frozen trajectories.}
\label{tab:capability}
\end{table*}

\paragraph{The price of coarseness: channel ablation.}
Table~\ref{tab:strongagg-main} is consistent with Theorem~\ref{thm:frontier}:
aggregate observation either violates the $1\%$ \texttt{target} budget or
sharply reduces utility, whereas role observation satisfies the budget with
nonzero utility. As
the $\alpha\to 0$ endpoint of Proposition~\ref{prop:ifc}, block-all
(CaMeL-inspired) baselines
achieve at most $1\%$ ASR but require at least $44.8\%$ over-intervention
(Table~\ref{tab:capability}).

\paragraph{Distribution shift.}
We freeze thresholds calibrated on a source distribution and evaluate them on
eight held-out attack and model-transfer conditions. Each condition contains
$725$-$1{,}346$ test fields and uses a $2\%$ \texttt{target} budget.

Across $20$ seeds, per-field CRC satisfies the budget in $100\%\pm0\%$ of
shifted conditions, compared with $61\%\pm33\%$ for the PACT-inspired
per-argument threshold and $48\%\pm48\%$ for the FIDES-inspired labels, as shown
in Figure~\ref{fig:headline}b. Its worst-case empirical \texttt{target} violation
is $0\%$ in all eight conditions.

The PACT-inspired per-argument threshold averages $1.5\%$ violation and stays
within budget in six of the eight conditions, with paired Wilcoxon
$p=7.8\times 10^{-3}$.

\paragraph{What drives the gain?}
Table~\ref{tab:symmetric} separates stratification and conformal calibration
through a $2{\times}2$ ablation on the same shifted fields.

With thresholds frozen on the source distribution, stratification is the main
source of robustness. Fixed per-role control and per-field CRC both
achieve $100\%$ compliance and $0\%$ worst-case violation. Conformal
single-threshold aggregate control is $42.5$ percentage points lower in
compliance than per-field CRC, with $p=8.2\times 10^{-5}$.

\begin{table}[t]
\centering
\small
\setlength{\tabcolsep}{3pt}
\begin{tabular}{@{}lcc@{}}
\toprule
Method & Compl.\ \% & Worst viol.\ \% \\
\midrule
\multicolumn{3}{@{}l@{}}{\emph{Frozen on source}}\\
Fixed agg.\ (PACT), single
  & $65.6{\pm}28.2$ & $2.9{\pm}2.3$ \\
CRC aggregate, single
  & $57.5{\pm}18.3$ & $5.2{\pm}1.8$ \\
Fixed per-role
  & $\mathbf{100.0{\pm}0.0}$ & $\mathbf{0.0{\pm}0.0}$ \\
CRC per-role
  & $\mathbf{100.0{\pm}0.0}$ & $\mathbf{0.0{\pm}0.0}$ \\
\midrule
\multicolumn{3}{@{}l@{}}{\emph{Recalibrated on shifted split}}\\
Fixed per-role
  & $89.4{\pm}12.7$ & $2.9{\pm}2.6$ \\
CRC per-role
  & $97.5{\pm}6.4$  & $0.8{\pm}1.6$ \\
\bottomrule
\end{tabular}
\caption{$2{\times}2$ shift ablation (20 seeds; same fields as
  Table~\ref{tab:capability}). Compl.\ denotes shift compliance.
  \emph{Fixed agg.\ (PACT), single} uses one source-frozen threshold, unlike
  the per-argument proxy in Table~\ref{tab:capability} ($65.6\%$ vs.\ $61\%$).
  Stratification improves robustness with frozen thresholds; conformal
  recalibration restores certification under $Q_r$.}
\label{tab:symmetric}
\end{table}

After recalibration on a shifted split, conformal correction adds finite-sample
value. It raises compliance by $8.1$ percentage points, with
$p=2.8\times 10^{-2}$, and lowers worst-case violation from $2.9\%$ to $0.8\%$.
Under detector noise, it adds $6.2$ points, with $p=2.5\times 10^{-2}$.
Even after both methods are recalibrated, per-field CRC leads by $36.2$
percentage points, with $p=1.2\times 10^{-4}$.

\paragraph{Unseen tool suites.}
We calibrate on three AgentDojo suites and evaluate on the held-out fourth, whose
tools, argument names, and environment are disjoint. This produces $24$
model-by-suite conditions at a $2\%$ budget over $20$ seeds.

Frozen per-field CRC is compliant in all $24/24$ conditions, with $1.5\%$
worst-case \texttt{target} violation. The PACT-inspired fixed single-threshold
proxy is compliant in $14/24$,
aggregate CRC in $12/24$, and fixed per-role control in $19/24$
(Supplement~E, Table~E.1 and Fig.~E.6). Deployable reversion narrows the
utility gap (Supplement~E).

\paragraph{Non-verbatim attacks and gradual drift.}
When the attacker's exact text is removed, overlap-based \texttt{target}
detection falls from $0.76$ to $0.39$ ROC-AUC. A frozen verbatim threshold then
exceeds the budget,
reaching $4.8\%$ mean and up to $60\%$ worst-case violation, whereas recalibrated
per-field CRC stays within budget ($\le0.06\%$)
(Supplement~A, Fig.~A.1).

Under a gradual detector-noise ramp, the frozen PACT-inspired threshold rises
from $1.7\%$ to $5.7\%$ violation and exceeds the budget.
The fixed per-role threshold crosses the $2\%$ budget once
$\mathrm{TV}(P,Q)\gtrsim 0.41$.
Online recalibration remains between $0.1\%$ and $0.5\%$ across all nine stages,
consistent with Proposition~\ref{prop:shift} (Supplement~E,
Figs.~E.4-E.5).

\paragraph{Adaptive attacks and confidence allocation.}
Under adaptive selection of the least detectable successful injection from
unseen families, the PACT-inspired proxy reaches $3.6\%$-$6.5\%$
\texttt{target} violation (up to $3.25{\times}$ the $2\%$ limit) while
per-field CRC stays at $0.0\%$; under the independent
$D_{\mathrm{alloc}}/D_{\mathrm{cert}}$ protocol, uniform Bonferroni matches
water-filling on held-out utility.

\FloatBarrier
\section{Discussion and Limitations}

Semantic-role stratification is the main source of robustness,
while conformal recalibration supplies finite-sample validity when labeled
target-condition data are available. Our claims therefore fall into three levels:
theoretically certified under exchangeable calibration/test role populations;
empirically robust but not freshly certified under frozen transfer and unseen
suites; and re-certified after recalibration on labeled target-condition data.

These guarantees carry a utility cost: deployable value reversion achieves
$25.2\%$ utility versus $32.3\%$ for the PACT-inspired proxy. Rare roles may
require pooled certification; the guarantee covers field-value integrity, not
tool selection, omitted actions, call ordering, or leakage through low-risk
fields; and abstain/revert utility on recorded traces does not capture live
replanning after a block or revert.

\section{Conclusion}

Aggregate action-level control pays a \emph{price of coarseness}: it inflates a
rare role's violation rate by $1/p_r$ and shrinks its effective budget to
$\alpha p_r$. Role-stratified per-field CRC instead certifies risk at the
semantic-role level. The utility gap tracks this price while the method delivers
consistent role-specific compliance, certified under exchangeability or after
recalibration and improving with any per-field detector. Future work extends the
guarantee to tool selection, omitted actions, and call ordering, and validates
the method under interactive deployment.

\clearpage
\setcounter{secnumdepth}{1}
\begin{center}
{\Large\bf Technical Appendix}
\end{center}
\vspace{0.5em}
\appendix
\numberwithin{table}{section}
\numberwithin{figure}{section}
\numberwithin{equation}{section}

\noindent This supplement collects proofs (including the certification-granularity
frontier and the conformal IFC relaxation), per-model tables, additional
stress tests, and a notation summary for the main paper. Section letters match
the references in the main text (Supplement~A-G). Unless noted, all numbers are
from a \emph{controlled trace-replay evaluation} on real agent traces: $20$
random calibration/test splits of cached traces with no new model calls and no
post-intervention re-execution. The sixth model in the evaluation panel is
Qwen2.5-7B.

\section{Detector Details and Exchangeability}
\label{app:detector}

\paragraph{Primary per-argument scores.}
The main comparison uses two scores. Both are shifted and rescaled to $[0,1]$,
with larger values indicating stronger untrusted influence. A field is allowed
when $s_i\le\tau$. The calibration and decision procedure (Equation~(5) and
Algorithm~1 of the main paper) is identical for each. They differ only in what
information they use. The non-verbatim stress test below introduces three
further score variants (literal-removed provenance, character-$n$-gram
provenance, and a counterfactual clean-vs-injected overlap); we use one fixed
name for each score throughout.

\emph{Attacker-literal overlap diagnostic (label-coupled).} Computed post-hoc by
measuring lexical overlap between the emitted field and the injection goal (the
attacker literal). Because the violation label is defined by whether the field
follows the injection goal, this score reads the same attacker literal used to
define the label. Its ROC-AUC of $0.93$ therefore partly reflects access to
evaluation-specific attacker text. We use it only as an \emph{oracle-aided
diagnostic} (an upper bound on detector quality), not as a deployment default.
It is distinct from the counterfactual clean-vs-injected overlap below, which
does not read the attacker literal.

\emph{Annotation-free provenance (deployable default).} Both oracle-free (no clean
counterfactual) and annotation-free (never reads the ground-truth attacker literal):
\begin{equation}
s_i = \mathrm{sim}(x_i, u) - \mathrm{sim}(x_i, t),
\end{equation}
where $u$ is the untrusted context and $t$ is the trusted user prompt. The raw
similarity difference is shifted and rescaled to $[0,1]$. This is the score
available at deployment (target-role ROC-AUC $0.71$).

\emph{Comparability caveat.} ROC-AUCs measured under the standard (verbatim)
attack and under the non-verbatim attack use different field populations and are
not directly comparable. For example, annotation-free provenance scores $0.71$
under the standard attack and $0.74$ under the non-verbatim attack; the two
values reflect different conditions, not a change in the score.

\paragraph{Verbatim-reuse ablation.}
A ground-truth-literal ablation attributes most of the target-role signal in all
overlap-based scores to verbatim reuse of injected content (target AUC $0.76 \to
0.39$ once the literal is removed). Consequently, the guarantee we certify is on
field-value integrity under this detector, and improving non-verbatim
detection, which the calibration layer inherits directly, is the main driver of
higher utility. This matches our premise that detector gains are orthogonal to the
calibration layer.

\paragraph{Non-verbatim adaptive attacker and detector inheritance.}
Figure~\ref{fig:nonverbatim} summarizes the non-verbatim stress test. We consider
an attacker that places the authority-bearing value in a \texttt{target} field
without leaving it verbatim in readable content, corresponding to the
literal-removed condition above. Across six models, 20 seeds, and a $2\%$
\texttt{target} budget, we report three findings.
\emph{(1) Evasion.} A PACT-inspired threshold frozen on the verbatim condition
does not transfer reliably: mean \texttt{target} violation rises to $4.8\%$,
with a worst case of $60\%$. Per-field CRC recalibrated on the non-verbatim
condition restores the budget, reaching $0.04\%$ violation. When the score is
nearly uninformative, this requires $84\%$ over-intervention.
\emph{(2) Detector inheritance.} Across the tested scores, the same per-field
CRC layer keeps \texttt{target} violation at or below $0.06\%$, with at least
$99\%$ budget compliance. Detector quality mainly determines utility. Among deployable
scores, over-intervention falls from $84\%$ at $0.39$ ROC-AUC for
literal-removed provenance to $77\%$ at $0.74$ ROC-AUC for annotation-free
provenance. The counterfactual clean-vs-injected overlap reaches $0.76$ ROC-AUC
and $38\%$ over-intervention, but is reported only as an oracle-aided upper
bound. The attacker-literal overlap diagnostic reaches $0.93$ ROC-AUC but is
excluded because it reads the attacker literal used to define the violation
label.
\emph{(3) Stronger deployable detection.} Against a present-but-obfuscated
attacker, a character-$n$-gram provenance score recovers signal lost by exact
matching, raising ROC-AUC from $0.39$ to $0.70$. With the same calibration and
enforcement rule, per-field CRC holds the budget while reducing
over-intervention from $84\%$ to $64\%$. Thus, better non-verbatim detection
improves utility, while the calibration layer remains unchanged.

\begin{figure}[t]
  \centering
  \includegraphics[width=\columnwidth]{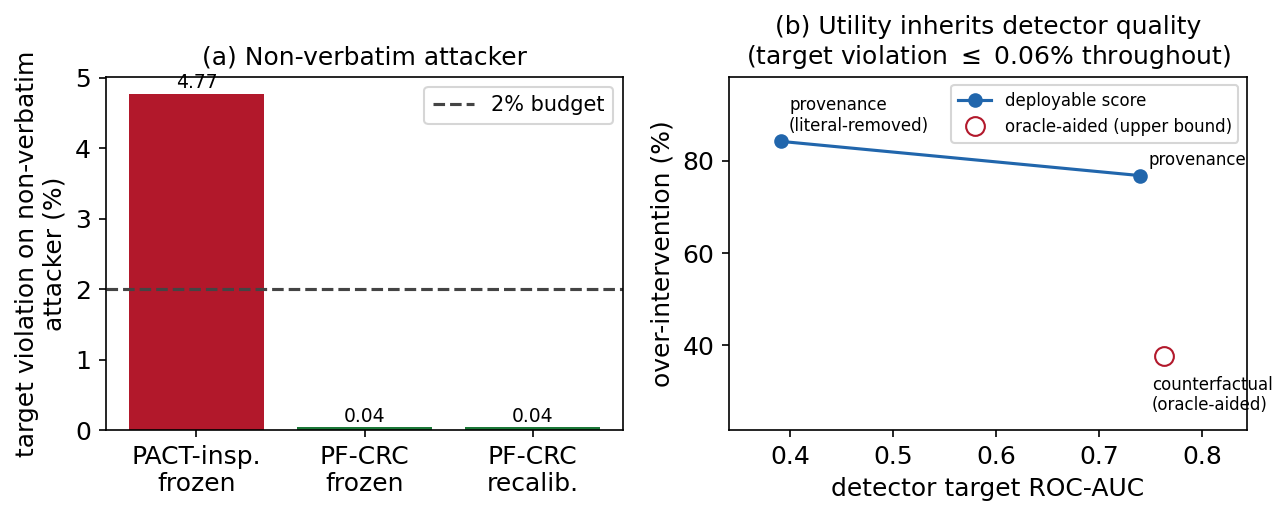}
  \caption{Non-verbatim attacks at a $2\%$ budget across six models and 20 seeds.
    The PACT-inspired frozen threshold exceeds budget, while per-field CRC holds.
    Over-intervention decreases as detector AUC improves.}
  \label{fig:nonverbatim}
\end{figure}

\paragraph{Field-level exchangeability.}
The calibration units are individual fields drawn from multiple episodes. The
guarantee therefore requires field-level exchangeability within each role
stratum, not episode-level exchangeability. This condition holds when fields of
the same role are drawn from i.i.d.\ episodes.

Our implemented calibration includes every emitted role-$r$ field, contributing
one calibration point per field rather than one per episode. Because each
per-role calibration set is built from role-$r$ fields alone, dependence between
fields of \emph{different} roles within the same episode does not affect any
single-role certificate. If an episode contributes several fields of the
\emph{same} role, however, those fields may be correlated, and counting all of
them can make $n_r$ overstate the effective sample size.

This distinction matters for the finite-sample certificate of
Theorem~2. Its Clopper-Pearson upper bound is \emph{exactly} binomial only when
the role-$r$ loss observations are independent Bernoulli draws; having
independent \emph{episodes} does not by itself make every individual field
independent when several same-role fields come from one episode. We therefore
state the assumption explicitly: the all-field certificates reported for the
main experiments hold under \emph{field-level independence} within each role
stratum. When multiple same-role fields per episode are correlated, an inflated
$n_r$ enters the denominator of the bound and can make the interval
anti-conservative, so the all-field numbers should be read as
field-independence-based rather than as episode-level exact certificates.

Two constructions remove this assumption by calibrating over independent
episodes: retaining one randomly selected role-$r$ field per episode, or using
the episode-level loss
$L_{e,r}(\tau)=\max_{i\in e:\, r_i=r} L_r(x_i;\tau)$,
which equals one when at least one role-$r$ field in episode $e$ is both
allowed and violated. Because each episode in our datasets contributes at most
one attacked (hence at most one \emph{violated}) field per role, the
certificate-relevant successes $k_r(\tau)$ are already at most one per episode;
the residual within-episode dependence therefore affects only the benign
denominator $n_r$. We recommend the episode-level construction whenever
deployments emit several same-role fields per episode, and we report the
one-per-episode sensitivity check alongside the all-field results in the code
archive.

\paragraph{Attack-conditional versus unconditional violation.}
\label{app:conditional}
The certified quantity is the unconditional role-specific violation rate
\[
V(r)=\Pr[\text{violated}\wedge\text{allowed}\mid\text{role}=r],
\]
whose denominator contains all role-$r$ fields, both benign and attacked. The
attack-conditional rate is
\[
V_{\mathrm{att}}(r)=\Pr[\text{violated}\wedge\text{allowed}\mid\text{role}=r,
\text{attacked}],
\]
By construction, an unattacked field cannot receive an attack-induced violation
label, so the clean-conditional rate $V_{\mathrm{clean}}(r)=0$ and the general
decomposition $V(r)=\pi(r)\,V_{\mathrm{att}}(r)+(1-\pi(r))\,V_{\mathrm{clean}}(r)$
reduces to
\[
V(r)=\pi(r)\,V_{\mathrm{att}}(r),
\]
where $\pi(r)$ is the field-level attack prevalence for role $r$.

Thus, at fixed $V_{\mathrm{att}}(r)$, the unconditional rate $V(r)$ decreases
linearly as attacks become rarer. We therefore report $V_{\mathrm{att}}(r)$
alongside $V(r)$, so the certificate is not interpreted as a
prevalence-invariant guarantee.

For the \texttt{target} role, the field stream is attack-heavy across all six
models, with $\pi(\texttt{target})=85\%$-$94\%$ and mean $91\%$
(Table~\ref{tab:conditional}). Accordingly, $V(\texttt{target})$ and
$V_{\mathrm{att}}(\texttt{target})$ are close. Per-field CRC keeps both at or
below $0.4\%$, whereas aggregate CRC reaches
$V_{\mathrm{att}}(\texttt{target})=5.5\%$.

In a low-prevalence deployment, aggregate control could satisfy an unconditional
budget while still allowing a large fraction of attacked fields. Calibrating only
on attacked role-$r$ fields instead gives a label-conditional conformal
certificate on $V_{\mathrm{att}}(r)$, which is invariant to test-time attack
prevalence $\pi(r)$.

\begin{table}[t]
\centering
\small
\setlength{\tabcolsep}{4pt}
\begin{tabular}{lccccc}
\toprule
 & & \multicolumn{2}{c}{Aggregate CRC} & \multicolumn{2}{c}{Per-field CRC} \\
\cmidrule(lr){3-4}\cmidrule(lr){5-6}
Model & $\pi$\,\% & $V$\,\% & $V_{\mathrm{att}}$\,\% & $V$\,\% & $V_{\mathrm{att}}$\,\% \\
\midrule
Gemini 2.5 Flash & 94.1 & 2.17 & 2.31 & 0.00 & 0.00 \\
Gemini 2.5 Pro   & 91.1 & 3.30 & 3.61 & 0.00 & 0.00 \\
GPT-4o-mini      & 93.2 & 3.44 & 3.70 & 0.00 & 0.00 \\
GPT-4o           & 92.4 & 5.02 & 5.45 & 0.00 & 0.00 \\
Llama 3.3 70B    & 89.6 & 4.24 & 4.71 & 0.33 & 0.38 \\
Qwen2.5-7B       & 84.9 & 1.11 & 1.30 & 0.00 & 0.00 \\
\bottomrule
\end{tabular}
\caption{Unconditional $V(\texttt{target})$ vs.\ attack-conditional
  $V_{\mathrm{att}}(\texttt{target})$ and field-level attack prevalence $\pi$
  ($\alpha=0.10$, 20 seeds). The attack-heavy field stream ($\pi=85$-$94\%$)
  makes $V\approx V_{\mathrm{att}}$, so per-field CRC's safety is not a
  rare-attack artifact. Reproduced by
  \texttt{experiments/attack\_conditional\_analysis.py}.}
\label{tab:conditional}
\end{table}

\section{The Aggregate-Budget Failure: Full Evidence}
\label{app:agg}

Figure~\ref{fig:diagnostic} shows \texttt{target}-role violation as a function of
the aggregate budget $\alpha$ on GPT-4o. Aggregate CRC satisfies the global
budget while \texttt{target} violation remains above its role-specific limit.
Per-field CRC instead controls the \texttt{target} role directly and remains at
zero observed \texttt{target} violation across the evaluated budgets.

Table~\ref{tab:panel} reports the corresponding six-model results using the
label-coupled diagnostic score, which serves as an oracle-aided upper bound on
detector quality. The same pattern holds across models: aggregate CRC can
satisfy its action-level objective while allowing excessive \texttt{target}-role
violations, whereas per-field CRC remains within the $1\%$ \texttt{target}
budget. Stronger single-threshold action losses (risk-weighted,
inverse-prevalence-weighted, and max-risk) are reported in the main paper; they
reduce dilution in some settings but do not remove the granularity mismatch.

\paragraph{InjecAgent replication.}
On InjecAgent with GPT-4o-mini under a $10\%$ aggregate budget, aggregate CRC
achieves $2.9\%$ overall field violation while allowing $100\%$ of attacked
\texttt{target} fields. Per-field CRC keeps \texttt{target} violation at $0\%$.
This secondary-benchmark pattern matches the AgentDojo dilution in
Figure~\ref{fig:diagnostic} and Table~\ref{tab:panel}.

\begin{figure}[t]
  \centering
  \includegraphics[width=\columnwidth]{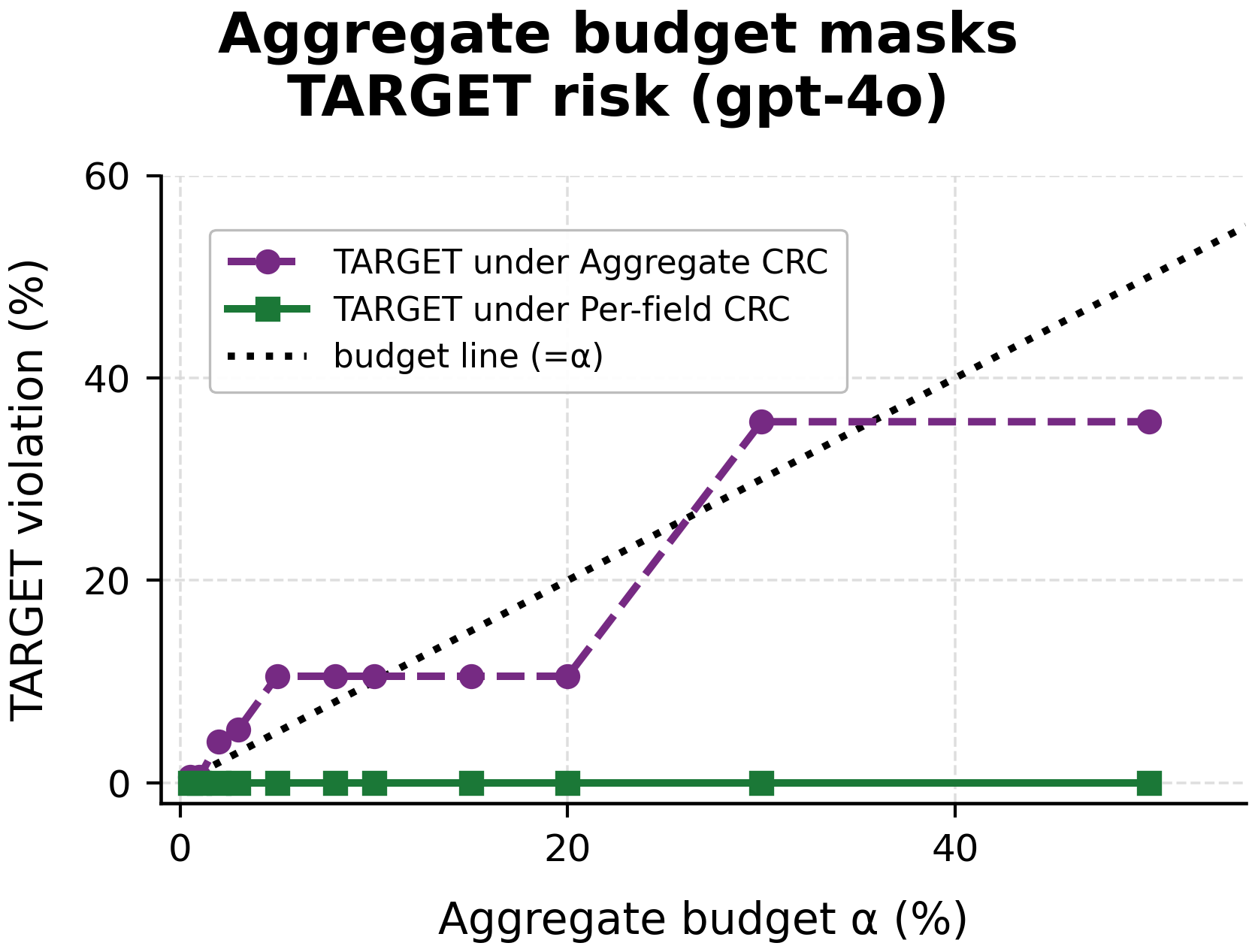}
  \caption{\texttt{target} violation vs.\ $\alpha$: aggregate CRC exceeds the role budget.
    Per-field CRC holds (GPT-4o, 20 seeds).}
  \label{fig:diagnostic}
\end{figure}

\begin{table*}[t]
\centering
\small
\setlength{\tabcolsep}{6pt}
\begin{tabular}{lcccccccc}
\toprule
& No def. & \multicolumn{4}{c}{\texttt{target} violation \%} & \multicolumn{3}{c}{Utility \%} \\
\cmidrule(lr){3-6}\cmidrule(lr){7-9}
Model & ASR \% & No def. & Agg.\ CRC & PACT$^\dagger$ & PF CRC & PACT$^\dagger$ & PF$_{\text{ab}}$ & PF$_{\text{rev}}$ \\
\midrule
Gemini 2.5 Flash & 49.1 & 43.8 & 2.0  & 1.3 & 0.0 & 25.7 & 7.7 & 27.9 \\
Gemini 2.5 Pro   & 23.5 & 26.3 & 4.0  & 0.0 & 0.0 & 34.7 & 14.5 & 27.2 \\
GPT-4o-mini      & 53.5 & 35.6 & 5.4  & 3.6 & 0.0 & 26.4 & 7.1 & 18.8 \\
GPT-4o           & 57.1 & 35.7 & 10.5 & 2.3 & 0.0 & 30.6 & 10.7 & 28.9 \\
Llama 3.3 70B    & 54.0 & 32.6 & 8.7  & 0.0 & 0.0 & 20.8 & 9.4 & 28.3 \\
Qwen2.5-7B       & 0.7  & 0.0  & 0.0  & 0.0 & 0.0 & 25.1 & 9.9 & 19.9 \\
\bottomrule
\end{tabular}
\caption{Per-model panel at a $1\%$ \texttt{target} budget using the
  label-coupled diagnostic score. This table is not a numerical expansion of
  Main Table~1 (channel ablation); it reports a separate per-model panel under
  the same score family. The two differ only in split protocol and reporting:
  Main Table~1 reports the per-field \texttt{target} violation \emph{averaged
  over $20$ random calibration/test splits}, so a model with occasional
  nonzero-violation splits (Llama~3.3~70B) shows a small positive mean
  ($0.0$-$0.3\%$), whereas this panel reports a \emph{single held-out split} and
  marks zero observed violations as $0.0\%$. Aggregate CRC exceeds budget; per-field CRC
  holds ($0\%$ = zero observed, Clopper-Pearson bound $\le 2.1\%$).
  $^\dagger$PACT-inspired fixed single-threshold proxy.
  No def.\ denotes the no-defense baseline; PF CRC is role-stratified per-field CRC;
  PF$_{\text{ab}}$/PF$_{\text{rev}}$ are its abstain/revert utilities.}
\label{tab:panel}
\end{table*}

\section{Certifiability and Deployable-Detector Results}
\label{app:cert}

Table~\ref{tab:certifiability} details certifiability for the \texttt{credential}
role. Across the $20$ group splits, the number of \texttt{credential}
calibration fields is small and varies substantially by model: from $2$-$6$
for Llama~3.3~70B to $8$-$22$ for GPT-4o, while some Qwen2.5-7B splits contain
no \texttt{credential} fields.

For each split, we compute the finite-sample floor $1/(n_r{+}1)$, then average
over the $20$ splits. The resulting per-model mean floors range from $6.7\%$ to
$63.6\%$. This range is across per-model averaged floors, not across individual
splits; the Qwen2.5-7B mean is increased by splits with $n_r=0$, whose floor is
$100\%$.

Because every per-model mean floor exceeds $1\%$, \texttt{credential} cannot be
certified individually at a $1\%$ budget. Pooling it with the other high-risk
roles produces a certifiable stratum, with floors of $0.14\%$-$0.49\%$ and
$0\%$ empirical \texttt{credential} violation. The certificate, however, applies
to the pooled high-risk average, not to \texttt{credential} alone. An individual
high-probability certificate would require approximately $300$
\texttt{credential} fields, compared with approximately $91$ available when
pooling across models.

Table~\ref{tab:deploy} compares the label-coupled diagnostic score with the
annotation-free provenance score. Although the deployable score has lower
ROC-AUC, $0.71$ versus $0.93$, it achieves similar safety: worst-case
\texttt{target} violation remains at or below $0.3\%$, with at least $95\%$
held-out compliance across all six models. This supports the main distinction
that conformal calibration determines risk control, while detector quality
primarily affects over-intervention and utility.

\begin{table}[t]
\centering
\small
\setlength{\tabcolsep}{3pt}
\begin{tabular}{lcccc}
\toprule
      & \multicolumn{2}{c}{\texttt{credential} alone} & \multicolumn{2}{c}{Pooled high-risk} \\
\cmidrule(lr){2-3}\cmidrule(lr){4-5}
Model & $n_{\text{cal}}$ range & Mean floor & Floor & Cert.@1\% (viol) \\
\midrule
Gemini 2.5 Flash & 3-18 & 10.9\% & 0.19\% & \cmark\ (0\%) \\
Gemini 2.5 Pro   & 2-14 & 15.1\% & 0.27\% & \cmark\ (0\%) \\
GPT-4o-mini      & 4-18 & 9.4\%  & 0.14\% & \cmark\ (0\%) \\
GPT-4o           & 8-22 & 6.7\%  & 0.20\% & \cmark\ (0\%) \\
Llama 3.3 70B    & 2-6  & 20.6\% & 0.27\% & \cmark\ (0\%) \\
Qwen2.5-7B       & 0-10 & 63.6\% & 0.49\% & \cmark\ (0\%) \\
\bottomrule
\end{tabular}
\caption{\texttt{credential} certifiability over $20$ group splits.
  $n_{\text{cal}}$ range is the actual integer count of \texttt{credential}
  calibration fields (minimum-maximum across the $20$ splits). Mean floor is
  the finite-sample floor $1/(n_r{+}1)$ computed on each split and then averaged
  over the $20$ splits. Alone it is $\gg 1\%$ for every model, so the
  $6.7\%$-$63.6\%$ reported range is a span of per-model averaged floors, not a
  span of individual splits. The pooled high-risk stratum is certifiable at
  $1\%$ (stratum average, not individual \texttt{credential}).}
\label{tab:certifiability}
\end{table}

\begin{table}[t]
\centering
\small
\setlength{\tabcolsep}{4pt}
\begin{tabular}{lcccccc}
\toprule
& \multicolumn{3}{c}{Label-coupled diag.} & \multicolumn{3}{c}{Deployable prov.} \\
\cmidrule(lr){2-4}\cmidrule(lr){5-7}
Model & viol.\ & o.-int.\ & compl.\ & viol.\ & o.-int.\ & compl.\ \\
\midrule
Gemini 2.5 Flash & 0.0 & 15.8 & 100 & 0.0 & 16.5 & 100 \\
Gemini 2.5 Pro   & 0.0 & 15.8 & 100 & 0.0 & 15.8 & 100 \\
GPT-4o-mini      & 0.0 & 17.0 & 100 & 0.0 & 16.3 & 100 \\
GPT-4o           & 0.0 & 25.0 & 100 & 0.0 & 26.1 & 100 \\
Llama 3.3 70B    & 0.3 & 15.8 & 95  & 0.0 & 15.6 & 100 \\
Qwen2.5-7B       & 0.0 & 8.8  & 100 & 0.2 & 8.4  & 95  \\
\bottomrule
\end{tabular}
\caption{Per-model deployable vs.\ label-coupled score ($2\%$ budget, 20 seeds):
  comparable safety despite lower ROC-AUC (0.71 vs.\ 0.93).
  Columns: \texttt{target} violation (\%), over-intervention (\%), and budget
  compliance (\%).}
\label{tab:deploy}
\end{table}

\section{Proofs}
\label{app:proofs}

Proposition~1 (concentration gap) is proved in the main paper. We give the remaining
proofs here.

\subsection{Proposition~2 (Aggregate control cannot certify a rare role)}
Recall that an \emph{aggregate-measurable} controller uses calibration data only
through the observables $(\Phi,L_{\mathrm{agg}})$ and does not observe per-role
violation labels.

\begin{proof}
Because the field-level prevalence in any finite benchmark is rational, write
$p_r=a/K$ with integers $1\le a\le K-1$; the special case $a=1$ recovers
$K=1/p_r$. Consider actions with $K$ fields, of which fields $1,\dots,a$ have
role~$r$ and fields $a{+}1,\dots,K$ have role \texttt{content}, so the field-level
prevalence of role~$r$ is exactly $p_r=a/K$. A fraction~$b$ of actions are bad and
the remainder are clean. The action score $S=\Phi(A)$ is drawn from the same
fixed distribution~$G$ for bad and clean actions, so the detector does not reveal
which field is violated.

Each bad action contains exactly one violated field, while each clean action
contains none. Under $P_1$, the violated field in every bad action is a role-$r$
field (say field~$1$). Under $P_0$, it is a \texttt{content} field (say
field~$K$). The two distributions have the same~$b$, the same score law~$G$, and
the same aggregate loss: every bad action that is allowed contributes
\[
L_{\mathrm{agg}}=\frac{1}{K}
\]
under both $P_0$ and $P_1$. Therefore, the joint distribution of
$(\Phi,L_{\mathrm{agg}})$ is identical under both distributions.

Because an aggregate-measurable controller~$\pi$ depends only on these
observables, its possibly randomized decision rule has the same distribution
under $P_0$ and $P_1$. Its admission decisions, benign utility, and aggregate
violation $\overline{V}(\pi)$ are therefore identical under both distributions.

Let
\[
q=\Pr[\text{allow}\mid\text{bad}]
\]
under~$\pi$. A bad allowed action contains exactly one violated-and-allowed field
among its~$K$ fields, so the aggregate rate is
\[
\overline{V}(\pi)=\frac{bq}{K}.
\]
Under $P_1$ that violated-and-allowed field is one of the $a$ role-$r$ fields per
action, and role-$r$ fields number $a$ per action, so
\[
V(r)=\frac{bq}{a}=\frac{K}{a}\,\overline{V}(\pi)=\frac{\overline{V}(\pi)}{p_r}.
\]
Under $P_0$, no role-$r$ field is ever violated, so $V(r)=0$.

Therefore, guaranteeing $V(r)\le\alpha$ under $P_1$ requires
\[
\overline{V}(\pi)\le\alpha\,p_r.
\]
Because the controller behaves identically under $P_0$, it incurs the same
utility cost there even though role~$r$ is never violated. By contrast,
role-stratified calibration observes the role labels and applies the standard
split-CRC guarantee directly to the $n_r$ calibration fields of role~$r$.
\end{proof}

\subsection{Theorem~1 (Certification-granularity frontier)}
\begin{proof}
\emph{(i) Achievability.} Fix a role~$r$ and restrict calibration to its $n_r$
fields together with one exchangeable test field. Because
\[
L_r(x;\tau)\in[0,1]
\]
is nondecreasing in~$\tau$, Equation~(5) of the main paper is the split
conformal risk control selector of~\citet{angelopoulos2024conformal}. Therefore,
\[
\mathbb{E}\bigl[L_r\bigl(X_{\mathrm{new}};\hat{\tau}(r)\bigr)\bigr]\le\alpha(r)
\]
whenever the feasible set is nonempty, equivalently when
\[
\alpha(r)\ge\frac{1}{n_r+1}.
\]
The selector uses only the field's role, score, and calibration violation label,
so it is measurable with respect to the role-observing channel
$\Phi_{\mathrm{role}}$. Applying the same construction separately to each role
certifies the full per-role budget vector.

\emph{(ii) Price of coarseness.} Proposition~2 constructs two distributions
$P_0$ and $P_1$ with the same joint law of $(\Phi,L_{\mathrm{agg}})$, but whose
role-$r$ violation rates differ by the factor $1/p_r$. An aggregate-measurable
controller cannot distinguish these distributions. Therefore, to certify
\[
V(r)\le\alpha
\]
under both, it must enforce
\[
\overline{V}\le\alpha\,p_r.
\]
By definition, the greatest benign utility available to any
$\Phi_{\mathrm{agg}}$-measurable controller satisfying this aggregate constraint
is
\[
u^\star(\alpha p_r).
\]
Thus, the factor~$p_r$ is the utility price of certifying a rare role through
aggregate observations.

\emph{(iii) Monotonicity.} Suppose $\Phi_1\preceq\Phi_2$, so that $\Phi_2$
refines $\Phi_1$. Every policy measurable with respect to $\Phi_1$ is also
measurable with respect to $\Phi_2$. Therefore, at any fixed utility level, the
set of certifiable budget vectors cannot shrink as the observation channel
becomes more informative:
\[
\Phi_{\mathrm{agg}}\preceq\Phi_{\mathrm{role}}\preceq\Phi_{\mathrm{field}}.
\]
This is Blackwell's comparison of experiments~\citep{blackwell1953equivalent}
applied to the risk-certification decision problem.
\end{proof}

\subsection{Proposition~3 (Conformal relaxation of noninterference)}
\begin{proof}
Fix an integrity-critical role~$r$. As $\alpha(r)\downarrow 0$, the constraint in
Equation~(5),
\[
\frac{1}{n_r+1}\Biggl(\sum_{i=1}^{n_r} L_r(x_i;\tau)+1\Biggr)\le\alpha(r),
\]
eventually admits no threshold~$\tau$ that allows any violated calibration field.
The selected threshold $\hat{\tau}(r)$ therefore falls below the smallest score
assigned to a violated field, so every field flagged by the detector is blocked.

In this zero-budget policy limit, the method recovers detector-relative
noninterference for role~$r$: untrusted influence cannot reach the
integrity-critical field whenever the detector identifies that influence. This
limit is conceptual, because finite-sample certification remains subject to the
floor
\[
\frac{1}{n_r+1}.
\]

For $\alpha(r)>0$, Equation~(5) instead selects the largest threshold whose
certified risk remains within $\alpha(r)$, thereby trading a controlled
residual-violation budget for utility. When
\[
\frac{1}{n_r+1}>\alpha(r),
\]
role~$r$ is not individually certifiable and must be pooled with other high-risk
roles or blocked, consistent with the zero-budget monitor.
\end{proof}

\subsection{Theorem~2 (Simultaneous stratum certificate)}
Let $\mathcal{R}_c=\{r:\alpha(r)<1\}$ and let
$g:\mathcal{R}_c\to\mathcal{G}$ map each controlled role to its final
calibration stratum ($g(r)=r$ when $r$ is individually certifiable; otherwise
$g(r)$ is the pooled high-risk stratum), with $\mathcal{G}=g(\mathcal{R}_c)$.
Recall that
\[
\hat{\tau}_g(\delta_g)=\sup\bigl\{\tau : U(k_g(\tau),n_g;\delta_g)\le\alpha(g)\bigr\},
\]
where $U(k,n;\delta_g)$ is the one-sided Clopper-Pearson upper confidence limit
and $k_g(\tau)$ is the number of stratum-$g$ calibration fields that are both
allowed and violated at threshold~$\tau$.

\begin{proof}
Fix a final stratum~$g\in\mathcal{G}$. Let
\[
\mathcal{T}_g=\{s_{(1)}\le\cdots\le s_{(n_g)}\}
\]
be the sorted unique calibration scores for that stratum, augmented with the
sentinels $s_{(0)}=-\infty$ and $s_{(n_g+1)}=+\infty$. Because $L_g(x;\tau)$ is
nondecreasing and right-continuous in~$\tau$, the empirical count $k_g(\tau)$
changes only at score values. Hence, $\hat{\tau}_g(\delta_g)$ is attained at a
candidate in $\mathcal{T}_g$.

Write $V(g)(\tau)=\Pr[\text{violated}\wedge\text{allowed}\mid\text{stratum}=g]$
for the population violation rate of stratum~$g$ at threshold~$\tau$. Because a
field is allowed exactly when its score satisfies $s\le\tau$, the map
$\tau\mapsto V(g)(\tau)$ is nondecreasing and right-continuous, the empirical
count $k_g(\tau)$ is nondecreasing in~$\tau$, and the retained candidate
$\hat{\tau}_g$ satisfies $U(k_g(\hat{\tau}_g),n_g;\delta_g)\le\alpha(g)$.

We do not treat any calibration-derived candidate as fixed; the candidate set
$\mathcal{T}_g$ is random because it depends on the calibration sample. Instead
we anchor the argument at the \emph{deterministic population boundary}
\[
\tau_g^{\star}=\inf\bigl\{\tau : V(g)(\tau)>\alpha(g)\bigr\},
\]
with $\tau_g^{\star}=+\infty$ if no such threshold exists. This quantity depends
only on the population law of the scores, not on the calibration sample, so
pointwise Clopper-Pearson validity applies at $\tau_g^{\star}$.

Suppose the selected threshold is unsafe, $V(g)(\hat{\tau}_g)>\alpha(g)$. Since
$V(g)$ is nondecreasing, the point $\hat{\tau}_g$ lies in
$\{\tau:V(g)(\tau)>\alpha(g)\}$, whose infimum is $\tau_g^{\star}$, so
$\hat{\tau}_g\ge\tau_g^{\star}$ and therefore $k_g(\hat{\tau}_g)\ge
k_g(\tau_g^{\star})$. Because $U(k,n;\delta_g)$ is nondecreasing in~$k$,
\[
U\bigl(k_g(\tau_g^{\star}),n_g;\delta_g\bigr)
\le U\bigl(k_g(\hat{\tau}_g),n_g;\delta_g\bigr)\le\alpha(g).
\]
Thus the unsafe event is contained in the event
$\{U(k_g(\tau_g^{\star}),n_g;\delta_g)\le\alpha(g)\}$, which is defined at the
fixed threshold $\tau_g^{\star}$.

If $V(g)(\tau_g^{\star})>\alpha(g)$, this event implies
$U(k_g(\tau_g^{\star}),n_g;\delta_g)<V(g)(\tau_g^{\star})$, a Clopper-Pearson
underestimate at a fixed threshold, whose probability is at most~$\delta_g$. If
right-continuity instead yields the boundary value
$V(g)(\tau_g^{\star})=\alpha(g)$, the unsafe event forces
$\hat{\tau}_g>\tau_g^{\star}$; applying the same containment at fixed thresholds
$t>\tau_g^{\star}$ with $V(g)(t)>\alpha(g)$ bounds each
$\{\hat{\tau}_g\ge t\}$ by~$\delta_g$, and taking the monotone limit
$t\downarrow\tau_g^{\star}$ preserves the bound. In either case,
\[
\Pr\bigl[V(g)(\hat{\tau}_g)>\alpha(g)\bigr]\le\delta_g.
\]
This argument uses only monotonicity and a fixed population boundary, so it
remains valid for discrete scores, ties, and flat regions of $V(g)(\tau)$.
Applying a union bound over the final strata $\mathcal{G}$ yields
\[
\Pr\bigl[\exists\, g\in\mathcal{G} : V(g)>\alpha(g)\bigr]
\le\sum_{g\in\mathcal{G}}\delta_g=\delta.
\]
Therefore, with probability at least $1-\delta$ over the calibration sample,
\[
V(g)\le\alpha(g)
\]
holds simultaneously for every final stratum $g\in\mathcal{G}$. Individually
retained roles ($g(r)=r$) receive role-specific certificates; pooled rare roles
receive only a pooled-stratum certificate.

If the confidence allocation $\{\delta_g\}$ is selected using an allocation set
$D_{\mathrm{alloc}}$ independent of the certification set $D_{\mathrm{cert}}$,
then conditional on $D_{\mathrm{alloc}}$ the allocation is fixed, and the same
argument applies to $D_{\mathrm{cert}}$.
\end{proof}

\subsection{Water-filling allocation (remark)}
When each stratum utility $u_g(\delta_g)$ is nondecreasing and concave, allocating
$\{\delta_g\}$ by equalizing marginal weighted utilities is a concave program
over the simplex $\sum_g\delta_g\le\delta$. The uniform Bonferroni split
$\delta_g=\delta/|\mathcal{G}|$ is always feasible. Empirically it matches
data-dependent water-filling under finite-sample utility estimates
(Supplement~E, Figure~\ref{fig:allocation}), so we use the uniform split as the
deployable default.

\subsection{Proposition~4 (Per-role degradation under shift)}
\begin{proof}
Fix a role~$r$ and threshold $\hat{\tau}(r)$ calibrated on the source
distribution~$P_r$. Because
\[
L_r(x;\tau)\in[0,1],
\]
the standard expectation-difference bound for bounded
functions~\citep{barber2023beyond} gives
\[
\bigl|\mathbb{E}_{Q_r}\bigl[L_r\bigl(X;\hat{\tau}(r)\bigr)\bigr]
-\mathbb{E}_{P_r}\bigl[L_r\bigl(X;\hat{\tau}(r)\bigr)\bigr]\bigr|
\le\mathrm{TV}(P_r,Q_r).
\]
Combining this with the source guarantee
\[
\mathbb{E}_{P_r}\bigl[L_r\bigl(X;\hat{\tau}(r)\bigr)\bigr]\le\alpha(r)
\]
yields
\[
\mathbb{E}_{Q_r}\bigl[L_r\bigl(X;\hat{\tau}(r)\bigr)\bigr]
\le\alpha(r)+\mathrm{TV}(P_r,Q_r).
\]
The same argument applies when $\hat{\tau}(r)$ is data-dependent by also taking
expectation over the calibration sample. Recalibrating on labeled samples from
$Q_r$ replaces $P_r$ with $Q_r$, removes the shift term, and restores a fresh
certificate under exchangeability with $Q_r$.

This bound is not itself a distribution-free certificate under shift unless
$\mathrm{TV}(P_r,Q_r)$ is known and small. Its role is to quantify how a frozen
per-role threshold can degrade and to motivate recalibration after distribution
shift.
\end{proof}

\section{Additional Stress Tests and Allocation}
\label{app:stress}

\paragraph{Detector noise.}
Figure~\ref{fig:degradation} varies the detector-noise level~$\sigma$. As noise
increases, the frozen PACT-inspired threshold drifts above the $2\%$
\texttt{target} budget. Per-field CRC instead recalibrates at each noise level
and achieves test-set budget compliance on $95\%$-$100\%$ of the sweep,
compared with $45\%$-$63\%$ for the PACT-inspired proxy.

Even when both methods are recalibrated at every noise level, per-field CRC
retains a $28.7$ percentage-point compliance advantage, with paired Wilcoxon
$p=4.8\times 10^{-8}$. This result shows that recalibration alone does not
remove the granularity mismatch: role stratification remains necessary for
reliable \texttt{target}-role control under detector noise.

\paragraph{Adaptive and unseen attacks.}
Figure~\ref{fig:adaptive} evaluates an adaptive attacker that selects the least
detectable successful injection and uses attack families unseen during
calibration. Under this setting, the frozen PACT-inspired threshold exceeds the
$2\%$ \texttt{target} budget, reaching $3.6\%$-$6.5\%$ \texttt{target}
violation, or up to $3.25\times$ the allowed limit. Per-field CRC remains at
$0.0\%$ \texttt{target} violation.

Attack novelty also increases evasion of the fixed threshold. Among violating
\texttt{target} fields, $10\%$ evade the threshold for the known attack family,
compared with $21\%$-$44\%$ for unseen attack families. These results show that
a fixed per-argument threshold is vulnerable to adaptive selection and
attack-family shift, whereas role-stratified per-field CRC continues to control
the \texttt{target} role.

\paragraph{Confidence-budget allocation.}
Figure~\ref{fig:allocation} compares uniform and data-dependent confidence
allocation. The allocation $\{\delta_g\}$ is selected using $D_{\mathrm{alloc}}$,
while thresholds and certificates are computed from the independent set
$D_{\mathrm{cert}}$. Utility is then evaluated on held-out test data.

The estimated stratum-utility curves are monotone and at least $99\%$ concave, so
the water-filling allocation applies. However, the uniform Bonferroni split
already achieves strong held-out utility. Data-dependent water-filling provides
no improvement, with pooled utility gaps of -3.1 percentage points at
$\delta{=}0.10$ and -1.8 percentage points at $\delta{=}0.05$.

These results support the uniform split as the deployable default: it preserves
the simultaneous certificate, avoids allocation overfitting, and performs
competitively under finite-sample utility estimates.

\paragraph{Role-label noise and calibration size.}
Under role-label noise, per-field CRC remains safer than the PACT-inspired
proxy. At noise rate $\eta{=}0.10$, \texttt{target} violation is $1.9\%$ for
per-field CRC versus $4.1\%$ for the proxy. At $\eta{=}0.30$, the corresponding
rates are $4.5\%$ and $9.7\%$. Although noisy role assignments weaken
stratification, per-field CRC continues to reduce \texttt{target} violation.

Calibration-size experiments confirm the finite-sample floor
\[
\frac{1}{n+1}.
\]
For example, when $n{=}5$, the smallest certifiable budget is $16.7\%$. This
forces rare roles such as \texttt{credential} to be pooled with other high-risk
roles rather than certified individually at a $1\%$ budget.

Using deployable value reversion, per-field CRC raises utility to
$19\%$-$29\%$ while maintaining $0\%$ observed \texttt{target} violation.

\begin{figure}[t]
  \centering
  \includegraphics[width=\columnwidth]{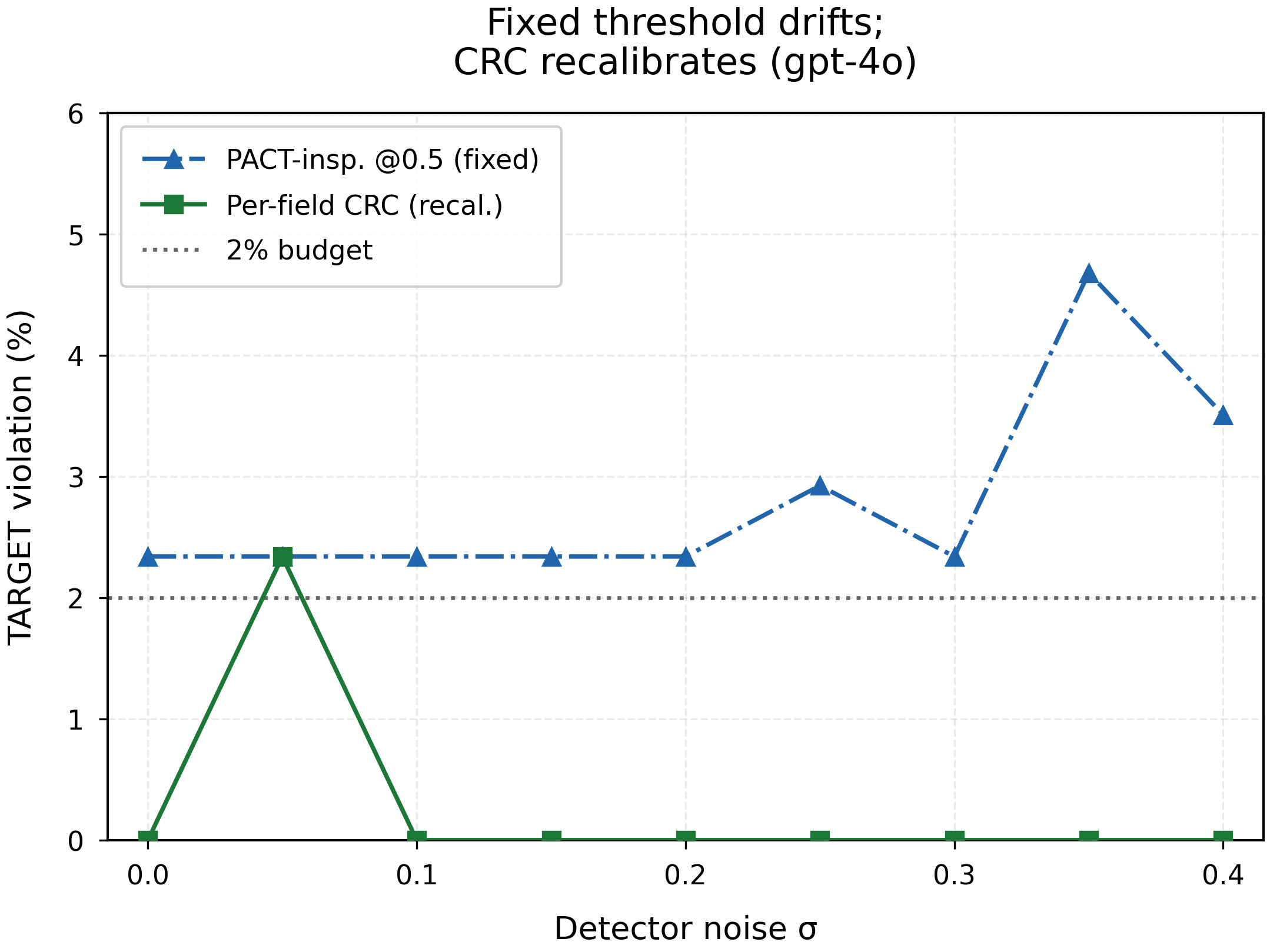}
  \caption{Detector noise: PACT-inspired fixed single-threshold drifts.
    Role-stratified per-field CRC recalibrates (GPT-4o).}
  \label{fig:degradation}
\end{figure}

\begin{figure}[t]
  \centering
  \includegraphics[width=\columnwidth]{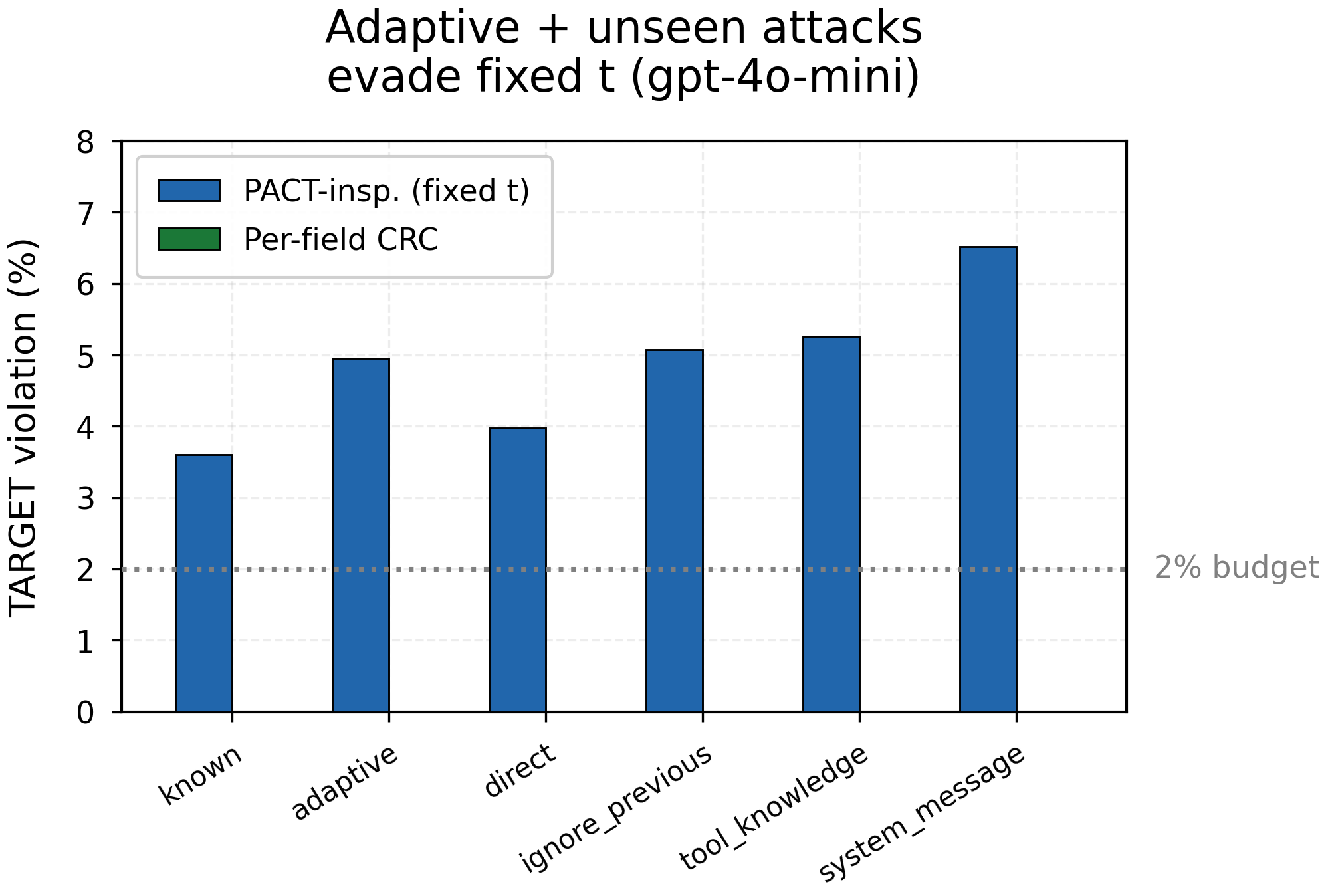}
  \caption{Adaptive/unseen attacks: fixed single threshold evaded.
    Role-stratified per-field CRC holds (GPT-4o-mini).}
  \label{fig:adaptive}
\end{figure}

\begin{figure}[t]
  \centering
  \includegraphics[width=\columnwidth]{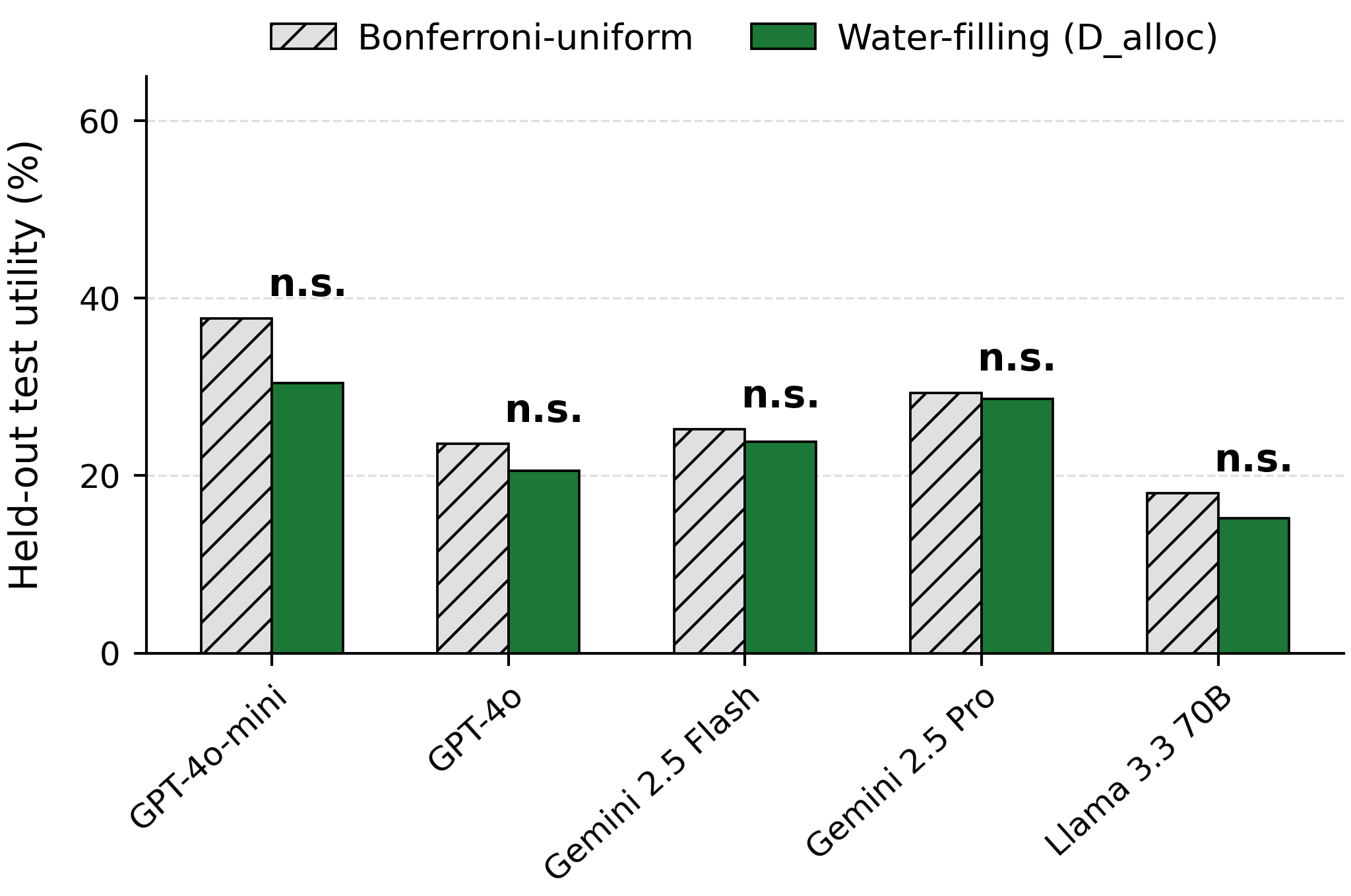}
  \caption{Uniform vs.\ data-dependent confidence-budget allocation under
    independent $D_{\mathrm{alloc}}/D_{\mathrm{cert}}$ ($\delta{=}0.10$, 20 seeds).
    Uniform is competitive on held-out utility.}
  \label{fig:allocation}
\end{figure}

\paragraph{Continual shift trajectory.}
Figure~\ref{fig:continual} evaluates deployment as a sequence of shifts rather
than a single test condition. We increase detector noise from $\sigma{=}0$ to
$0.40$ across nine stages, using four models, 20 seeds, and a $2\%$
\texttt{target} budget.

The frozen PACT-inspired threshold rises monotonically from $1.7\%$ to $5.7\%$
\texttt{target} violation and remains within budget at only $1$ of $9$ stages.
Fixed per-role control remains within budget throughout, with
$0.5\%$-$1.0\%$ violation across all nine stages. Per-field CRC with online
recalibration achieves the lowest and most stable worst-case violation,
remaining between $0.1\%$ and $0.5\%$ at every stage.

At high noise, online recalibration reduces violation to about half that of the
fixed per-role threshold. These results show that role stratification
provides robustness across the full shift trajectory, while conformal
recalibration further limits the accumulation of error as detector drift
increases.

\begin{figure}[t]
  \centering
  \includegraphics[width=\columnwidth]{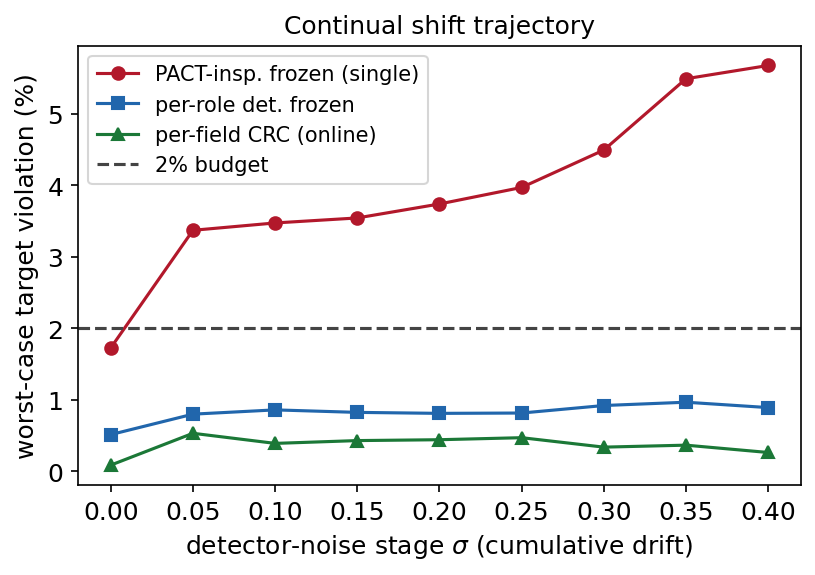}
  \caption{Detector-noise drift from $\sigma{=}0$ to $0.40$ at a $2\%$ budget
    across four models and 20 seeds. The frozen PACT-inspired threshold rises
    above budget, while recalibrated per-field CRC remains stable.}
  \label{fig:continual}
\end{figure}

\begin{figure}[t]
  \centering
  \includegraphics[width=\columnwidth]{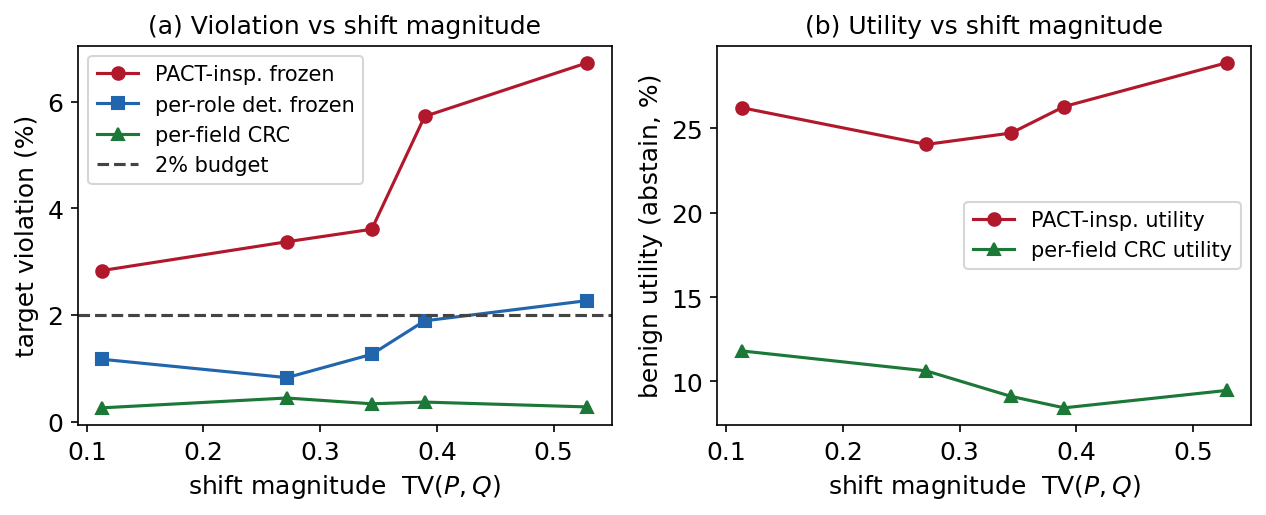}
  \caption{When to recalibrate vs.\ $\mathrm{TV}(P,Q)$ (four models, 20 seeds).
    (a)~Violation: frozen constants grow with TV. Per-field CRC flat.
    (b)~Utility cost (abstain).}
  \label{fig:crossover}
\end{figure}

\paragraph{Utility-safety crossover versus shift magnitude.}
Figure~\ref{fig:crossover} relates performance to the measured
\texttt{target}-role distribution shift $\mathrm{TV}(P,Q)$. As predicted by
Proposition~4, violation under frozen thresholds increases with shift
magnitude. The Pearson correlation between $\mathrm{TV}(P,Q)$ and
\texttt{target} violation is $0.19$ for the PACT-inspired proxy and $0.10$ for
the fixed per-role threshold.

The fixed per-role threshold crosses the $2\%$ \texttt{target} budget
once $\mathrm{TV}(P,Q)\gtrsim 0.41$. Per-field CRC with recalibration remains
nearly flat, with correlation approximately zero and \texttt{target} violation
near $0.3\%$ across the full noise ramp.

This crossover provides a practical recalibration rule: estimate
$\mathrm{TV}(P,Q)$ from re-scored traces and recalibrate when the measured shift
approaches the point at which the frozen threshold exceeds budget. This avoids
the utility cost of recalibration when the shift is small while restoring a
fresh certificate when the source threshold is no longer reliable.

\paragraph{Leave-one-suite-out transfer.}
Table~\ref{tab:loso} and Figure~\ref{fig:loso-supp} evaluate transfer to unseen
AgentDojo suites. We calibrate on three suites and test on the held-out fourth,
whose tools, argument names, and environment are disjoint from calibration.
This produces $24$ model-by-suite conditions at a $2\%$ \texttt{target} budget
over $20$ seeds.

Frozen role-stratified per-field CRC is compliant in all $24$ conditions, with
$1.5\%$ worst-case \texttt{target} violation. The frozen PACT-inspired fixed
single-threshold proxy is compliant in $14/24$ conditions, aggregate CRC in
$12/24$, and the fixed per-role threshold in $19/24$.

These results show that role stratification transfers more reliably than
single-threshold control to unseen tool suites. However, the source conformal
certificate does not extend to the held-out suite; the reported results are
empirical compliance under transfer, not a fresh certificate.

\begin{table}[t]
\centering
\small
\setlength{\tabcolsep}{4pt}
\begin{tabular}{@{}lccc@{}}
\toprule
Defense (frozen) & Compl.\ & Worst viol.\ & Over \\
\midrule
PACT-insp.\ (single) & 14/24 & 5.8 & 9.3 \\
Aggregate CRC        & 12/24 & 5.0 & 38.6 \\
Fixed per-role       & 19/24 & 3.8 & 33.0 \\
Per-field CRC        & 24/24 & 1.5 & 42.2 \\
\bottomrule
\end{tabular}
\caption{Leave-one-suite-out results across $24$ conditions at a $2\%$ budget.
  Compl.\ = test compliance; Worst viol.\ = worst \texttt{target} violation \%;
  Over = over-intervention \%. Role-stratified per-field CRC is compliant in all
  $24$ conditions. Source certificates do not extend to the unseen suite.
  PACT-insp.\ denotes the controlled fixed single-threshold proxy.}
\label{tab:loso}
\end{table}

\begin{figure}[t]
  \centering
  \includegraphics[width=\columnwidth]{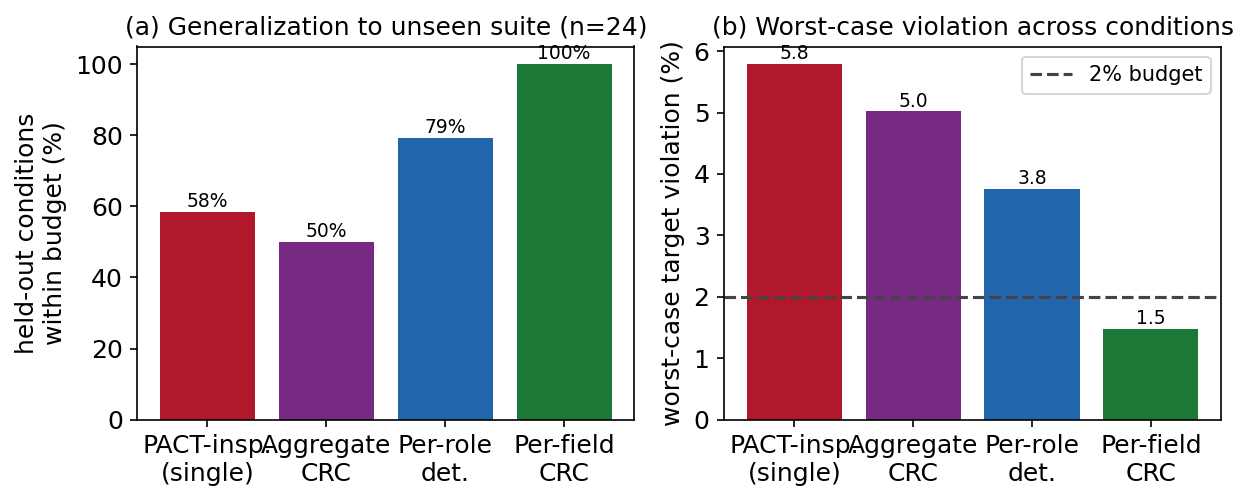}
  \caption{Leave-one-suite-out ($24$ conditions, $2\%$ budget): per-field CRC
    $24/24$. Frozen thresholds inconsistent.}
  \label{fig:loso-supp}
\end{figure}

\section{Extended Limitations and Reproducibility}
\label{app:repro}

Because we score interventions on recorded traces, abstain/revert utility does
not capture live replanning after a block. PACT-/FIDES-/CaMeL-inspired
comparisons isolate enforcement granularity under a shared detector and are not
full reproductions. The code archive provides splits, hyperparameters, seeds,
versions, raw JSON, and a script-to-experiment map; all reported experiments run
from cached traces without API access.

\paragraph{Computing infrastructure.}
Reproduction uses cached traces on CPU (Ubuntu~22.04.5, Python~3.10.18;
\texttt{requirements.txt}). Development host: Intel~i9-13900K, 62\,GiB RAM,
RTX~4090 (CUDA~12.2). Trace generation used hosted LLM APIs; regenerating
traces is optional and not required to reproduce the paper's statistics.

\section{Notations}
\label{app:notations}

\noindent Table~\ref{tab:notations} lists the main symbols used in the paper and
this supplement (role names such as \texttt{target} are omitted).

\begin{center}
\scriptsize
\setlength{\tabcolsep}{2.5pt}
\renewcommand{\arraystretch}{0.88}
\begin{tabular}{@{}lp{0.72\columnwidth}@{}}
\toprule
Symbol & Meaning \\
\midrule
\multicolumn{2}{@{}l}{\emph{Actions, roles, and scores}} \\
$A=(\textit{op},\,x_1,\dots,x_k)$ &
  Structured tool-call action \\
$x_i$ / $x$ &
  Named argument (field); enforcement unit \\
$r(i)$, $r$, $\mathcal{R}$ &
  Role of argument $i$; role set \\
$s_i$ / $s(x)$ &
  Nonconformity score (larger $\Rightarrow$ more untrusted influence) \\
$\tau$ &
  Allow threshold ($s_i\le\tau$) \\
$\mathrm{sim}$, $u$, $t$ &
  Similarity; untrusted context; trusted prompt \\
\midrule
\multicolumn{2}{@{}l}{\emph{Risks and budgets}} \\
$V(r)$ &
  $\Pr[\text{violated}\wedge\text{allowed}\mid\text{role}=r]$ \\
$V_{\mathrm{att}}(r)$, $V_{\mathrm{clean}}(r)$ &
  Attack-/clean-conditional rates ($V_{\mathrm{clean}}=0$) \\
$\pi(r)$ &
  Attack prevalence ($V=\pi V_{\mathrm{att}}$) \\
$\alpha(r)$, $\alpha$, $\alpha_{\mathrm{agg}}$ &
  Role-specific, generic, and aggregate risk budgets \\
$p_r$, $\overline{V}$ &
  Role prevalence; aggregate harm $\sum_r p_r V(r)$ \\
$p_{\mathrm{cred}\mid\mathrm{pool}}$ &
  \texttt{credential} prevalence in the pooled high-risk stratum \\
\midrule
\multicolumn{2}{@{}l}{\emph{Calibration and CRC}} \\
$L_r(x;\tau)$ &
  $\mathbf{1}\{s(x)\le\tau\wedge\text{$x$ is violated}\}$ \\
$L_{e,r}(\tau)$ &
  Episode-level loss $\max_{i\in e:\,r_i=r} L_r(x_i;\tau)$ \\
$n_r$, $1/(n_r{+}1)$ &
  Role-$r$ calibration size; certifiability floor \\
$\hat{\tau}(r)$ &
  Largest CRC threshold with inflated risk $\le\alpha(r)$ \\
$X_{\mathrm{new}}$, $y_i$ &
  Exchangeable test field; violation label \\
$g(r)$ &
  Final stratum ($r$ if certifiable, else pooled high-risk) \\
$L_{\mathrm{agg}}(A)$ &
  Action-average violation loss \\
\midrule
\multicolumn{2}{@{}l}{\emph{Simultaneous certificate}} \\
$\mathcal{G}$ &
  Final calibration strata $\mathcal{G}=g(\mathcal{R}_c)$ \\
$k_g(\tau)$ &
  Allowed-and-violated stratum-$g$ calibration count at $\tau$ \\
$U(k_g(\tau),n_g;\delta_g)$ &
  Exact one-sided Clopper-Pearson upper bound \\
$\hat{\tau}_g(\delta_g)$ &
  High-probability threshold with CP bound $\le\alpha(g)$ \\
$\delta_g$, $\delta$, $\mathcal{R}_c$ &
  Per-stratum / total failure prob.; controlled roles \\
$D_{\mathrm{alloc}}$, $D_{\mathrm{cert}}$ &
  Allocation and certification splits \\
$\tau_g^{\star}$ &
  $\inf\{\tau:V(g)(\tau)>\alpha(g)\}$ \\
\midrule
\multicolumn{2}{@{}l}{\emph{Observation channels and shift}} \\
$\Phi(A)$ &
  Score information visible for action $A$ \\
$\Phi_{\mathrm{agg}}$, $\Phi_{\mathrm{role}}$, $\Phi_{\mathrm{field}}$ &
  Aggregate, role, and field observation channels \\
$\Phi_1\preceq\Phi_2$, $u^\star(\beta)$ &
  Channel refinement; max utility under $\overline{V}\le\beta$ \\
$P_r$, $Q_r$, $\mathrm{TV}(P_r,Q_r)$ &
  Source/test role distributions; total variation \\
\bottomrule
\end{tabular}
\end{center}
\vspace{-0.5em}
\captionof{table}{Main notations used throughout the paper and this supplement.}
\label{tab:notations}

\vspace{-0.35em}

\bibliography{references}

@article{pact2026,
  title        = {The Granularity Mismatch in Agent Security: Argument-Level
                  Provenance Solves Enforcement and Isolates the {LLM}
                  Reasoning Bottleneck},
  author       = {Fan, Linfeng and Li, Ziwei and Tian, Yuan and Wang, Yichen
                  and Li, Rongsheng and Wang, Xiong},
  journal      = {arXiv preprint arXiv:2605.11039},
  year         = {2026},
  url          = {https://arxiv.org/abs/2605.11039}
}

@article{fides2025,
  title        = {Securing AI Agents with Information-Flow Control},
  author       = {Costa, Manuel and K{\"o}pf, Boris and Kolluri, Aashish and
                  Paverd, Andrew and Russinovich, Mark and Salem, Ahmed and
                  Tople, Shruti and Wutschitz, Lukas and
                  Zanella-B{\'e}guelin, Santiago},
  journal      = {arXiv preprint arXiv:2505.23643},
  year         = {2025},
  url          = {https://arxiv.org/abs/2505.23643}
}

@article{camel2025,
  title        = {Defeating Prompt Injections by Design},
  author       = {Debenedetti, Edoardo and Shumailov, Ilia and Fan, Tianqi and
                  Hayes, Jamie and Carlini, Nicholas and Fabian, Daniel and
                  Kern, Christoph and Shi, Chongyang and Terzis, Andreas and
                  Tram{\`e}r, Florian},
  journal      = {arXiv preprint arXiv:2503.18813},
  year         = {2025},
  url          = {https://arxiv.org/abs/2503.18813}
}

@article{arm2026,
  title        = {Causality Laundering: Denial-Feedback Leakage in
                  Tool-Calling {LLM} Agents},
  author       = {Chinaei, Mohammad Hossein},
  journal      = {arXiv preprint arXiv:2604.04035},
  year         = {2026},
  url          = {https://arxiv.org/abs/2604.04035}
}

@article{agentvisor2026,
  title        = {{AgentVisor}: Defending {LLM} Agents Against Prompt Injection
                  via Semantic Virtualization},
  author       = {Ying, Zonghao and Wang, Haozheng and Liu, Jiangfan and
                  Zou, Quanchen and Liu, Aishan and Yang, Jian and
                  Yang, Yaodong and Liu, Xianglong},
  journal      = {arXiv preprint arXiv:2604.24118},
  year         = {2026},
  url          = {https://arxiv.org/abs/2604.24118}
}

@article{ipisok2025,
  title        = {Taxonomy, Evaluation and Exploitation of {IPI}-Centric {LLM}
                  Agent Defense Frameworks},
  author       = {Ji, Zimo and Wang, Xunguang and Li, Zongjie and
                  Ma, Pingchuan and Gao, Yudong and Wu, Daoyuan and
                  Yan, Xincheng and Tian, Tian and Wang, Shuai},
  journal      = {arXiv preprint arXiv:2511.15203},
  year         = {2025},
  url          = {https://arxiv.org/abs/2511.15203}
}

@inproceedings{struq2025,
  title        = {{StruQ}: Defending Against Prompt Injection with Structured
                  Queries},
  author       = {Chen, Sizhe and Piet, Julien and Sitawarin, Chawin and
                  Wagner, David},
  booktitle    = {34th USENIX Security Symposium (USENIX Security 25)},
  pages        = {2383--2400},
  year         = {2025},
  url          = {https://www.usenix.org/conference/usenixsecurity25/presentation/chen-sizhe}
}

@inproceedings{secalign2025,
  title        = {{SecAlign}: Defending Against Prompt Injection with
                  Preference Optimization},
  author       = {Chen, Sizhe and Zharmagambetov, Arman and Mahloujifar,
                  Saeed and Chaudhuri, Kamalika and Wagner, David and
                  Guo, Chuan},
  booktitle    = {Proceedings of the 2025 ACM SIGSAC Conference on Computer
                  and Communications Security (CCS)},
  pages        = {2833-2847},
  year         = {2025},
  doi          = {10.1145/3719027.3744836},
  url          = {https://arxiv.org/abs/2410.05451}
}

@inproceedings{spotlighting2024,
  title        = {Defending Against Indirect Prompt Injection Attacks With
                  Spotlighting},
  author       = {Hines, Keegan and Lopez, Gary and Hall, Matthew and
                  Zarfati, Federico and Zunger, Yonatan and Kiciman, Emre},
  booktitle    = {Proceedings of the Conference on Applied Machine Learning
                  in Information Security (CAMLIS 2024)},
  series       = {CEUR Workshop Proceedings},
  volume       = {3920},
  pages        = {48--62},
  year         = {2024},
  publisher    = {CEUR-WS.org},
  url          = {https://ceur-ws.org/Vol-3920/paper03.pdf}
}

@article{instructionhierarchy2024,
  title        = {The Instruction Hierarchy: Training {LLMs} to Prioritize
                  Privileged Instructions},
  author       = {Wallace, Eric and Xiao, Kai and Leike, Reimar and
                  Weng, Lilian and Heidecke, Johannes and Beutel, Alex},
  journal      = {arXiv preprint arXiv:2404.13208},
  year         = {2024},
  url          = {https://arxiv.org/abs/2404.13208}
}

@inproceedings{jatmo2024,
  title        = {{Jatmo}: Prompt Injection Defense by Task-Specific
                  Finetuning},
  author       = {Piet, Julien and Alrashed, Maha and Sitawarin, Chawin and
                  Chen, Sizhe and Wei, Zeming and Sun, Elizabeth and
                  Alomair, Basel and Wagner, David A.},
  booktitle    = {Computer Security -- {ESORICS} 2024, Part {I}},
  series       = {Lecture Notes in Computer Science},
  volume       = {14982},
  pages        = {105--124},
  year         = {2024},
  publisher    = {Springer},
  doi          = {10.1007/978-3-031-70879-4_6}
}

@inproceedings{liu2024formalizing,
  title        = {Formalizing and Benchmarking Prompt Injection Attacks and
                  Defenses},
  author       = {Liu, Yupei and Jia, Yuqi and Geng, Runpeng and Jia, Jinyuan
                  and Gong, Neil Zhenqiang},
  booktitle    = {33rd USENIX Security Symposium (USENIX Security 24)},
  pages        = {1831-1847},
  year         = {2024},
  url          = {https://www.usenix.org/conference/usenixsecurity24/presentation/liu-yupei}
}

@article{cora2026,
  title        = {{CORA}: Conformal Risk-Controlled Agents for Safeguarded
                  Mobile {GUI} Automation},
  author       = {Feng, Yushi and Du, Junye and Wang, Qifan and Ma, Zizhan and
                  Niu, Qian and Matsuo, Yutaka and Feng, Long and Yu, Lequan},
  journal      = {arXiv preprint arXiv:2604.09155},
  year         = {2026},
  url          = {https://arxiv.org/abs/2604.09155}
}

@misc{conformalabstention2024,
  title         = {Mitigating {LLM} Hallucinations via Conformal Abstention},
  author        = {Abbasi-Yadkori, Yasin and Kuzborskij, Ilja and Stutz, David
                   and Gy{\"o}rgy, Andr{\'a}s and Fisch, Adam and
                   Doucet, Arnaud and Beloshapka, Iuliya and Weng, Wei-Hung and
                   Yang, Yao-Yuan and Szepesv{\'a}ri, Csaba and
                   Cemgil, Ali Taylan and Tomasev, Nenad},
  year          = {2024},
  eprint        = {2405.01563},
  archivePrefix = {arXiv}
}

@article{conformalactuator2025,
  title        = {Taming Variability: Randomized and Bootstrapped Conformal
                  Risk Control for {LLMs}},
  author       = {Pang, Lingyou and Huang, Lei and Lin, Jianyu and
                  Wang, Tianyu and Aue, Alexander and Priebe, Carey E.},
  journal      = {arXiv preprint arXiv:2509.23007},
  year         = {2025},
  url          = {https://arxiv.org/abs/2509.23007}
}

@inproceedings{angelopoulos2024conformal,
  title        = {Conformal Risk Control},
  author       = {Angelopoulos, Anastasios N. and Bates, Stephen and
                  Fisch, Adam and Lei, Lihua and Schuster, Tal},
  booktitle    = {International Conference on Learning Representations (ICLR)},
  year         = {2024},
  url          = {https://arxiv.org/abs/2208.02814}
}

@article{angelopoulos2023gentle,
  title        = {Conformal Prediction: A Gentle Introduction},
  author       = {Angelopoulos, Anastasios N. and Bates, Stephen},
  journal      = {Foundations and Trends in Machine Learning},
  volume       = {16},
  number       = {4},
  pages        = {494--591},
  year         = {2023},
  doi          = {10.1561/2200000101}
}

@article{angelopoulos2025ltt,
  title        = {Learn then Test: Calibrating Predictive Algorithms to
                  Achieve Risk Control},
  author       = {Angelopoulos, Anastasios N. and Bates, Stephen and
                  Cand{\`e}s, Emmanuel J. and Jordan, Michael I. and
                  Lei, Lihua},
  journal      = {The Annals of Applied Statistics},
  volume       = {19},
  number       = {2},
  pages        = {1641-1662},
  year         = {2025},
  doi          = {10.1214/24-AOAS1998},
  url          = {https://doi.org/10.1214/24-AOAS1998}
}

@article{bates2021rcps,
  title        = {Distribution-Free, Risk-Controlling Prediction Sets},
  author       = {Bates, Stephen and Angelopoulos, Anastasios and
                  Lei, Lihua and Malik, Jitendra and Jordan, Michael I.},
  journal      = {Journal of the ACM},
  volume       = {68},
  number       = {6},
  pages        = {43:1--43:34},
  year         = {2021},
  doi          = {10.1145/3478535},
  url          = {https://arxiv.org/abs/2101.02703}
}

@book{vovk2005algorithmic,
  title        = {Algorithmic Learning in a Random World},
  author       = {Vovk, Vladimir and Gammerman, Alexander and Shafer, Glenn},
  publisher    = {Springer},
  year         = {2005},
  url          = {https://link.springer.com/book/10.1007/978-0-387-25061-8}
}

@inproceedings{papadopoulos2002inductive,
  title        = {Inductive Confidence Machines for Regression},
  author       = {Papadopoulos, Harris and Proedrou, Kostas and
                  Vovk, Volodya and Gammerman, Alex},
  booktitle    = {Machine Learning: ECML 2002},
  series       = {Lecture Notes in Computer Science},
  volume       = {2430},
  pages        = {345--356},
  year         = {2002},
  publisher    = {Springer},
  doi          = {10.1007/3-540-36755-1_29}
}

@article{lei2018distribution,
  title        = {Distribution-Free Predictive Inference for Regression},
  author       = {Lei, Jing and G'Sell, Max and Rinaldo, Alessandro and
                  Tibshirani, Ryan J. and Wasserman, Larry},
  journal      = {Journal of the American Statistical Association},
  volume       = {113},
  number       = {523},
  pages        = {1094-1111},
  year         = {2018},
  doi          = {10.1080/01621459.2017.1307116},
  url          = {https://doi.org/10.1080/01621459.2017.1307116}
}

@techreport{vovk2003mondrian,
  title        = {Mondrian Confidence Machine},
  author       = {Vovk, Vladimir and Lindsay, David and Nouretdinov, Ilia and
                  Gammerman, Alex},
  institution  = {Royal Holloway, University of London},
  year         = {2003},
  note         = {On-line Compression Modelling Project, Working Paper \#4},
  url          = {https://alrw.cs.rhul.ac.uk/old/04.pdf}
}

@inproceedings{ding2023classconditional,
  title        = {Class-Conditional Conformal Prediction with Many Classes},
  author       = {Ding, Tiffany and Angelopoulos, Anastasios N. and
                  Bates, Stephen and Jordan, Michael I. and
                  Tibshirani, Ryan J.},
  booktitle    = {Advances in Neural Information Processing Systems
                  (NeurIPS)},
  year         = {2023},
  url          = {https://arxiv.org/abs/2306.09335}
}

@article{barber2023beyond,
  title        = {Conformal Prediction Beyond Exchangeability},
  author       = {Barber, Rina Foygel and Cand{\`e}s, Emmanuel J. and
                  Ramdas, Aaditya and Tibshirani, Ryan J.},
  journal      = {The Annals of Statistics},
  volume       = {51},
  number       = {2},
  pages        = {816-845},
  year         = {2023},
  doi          = {10.1214/23-AOS2276},
  url          = {https://doi.org/10.1214/23-AOS2276}
}

@inproceedings{bairaktari2025kandinsky,
  title        = {Kandinsky Conformal Prediction: Beyond Class- and
                  Covariate-Conditional Coverage},
  author       = {Bairaktari, Konstantina and Wu, Jiayun and Wu, Steven},
  booktitle    = {Proceedings of the 42nd International Conference on Machine
                  Learning},
  volume       = {267},
  series       = {Proceedings of Machine Learning Research},
  pages        = {2581--2602},
  year         = {2025},
  publisher    = {PMLR},
  url          = {https://proceedings.mlr.press/v267/bairaktari25a.html}
}

@inproceedings{debenedetti2024agentdojo,
  title        = {{AgentDojo}: A Dynamic Environment to Evaluate Prompt
                  Injection Attacks and Defenses for {LLM} Agents},
  author       = {Debenedetti, Edoardo and Zhang, Jie and Balunovi{\'c},
                  Mislav and Beurer-Kellner, Luca and Fischer, Marc and
                  Tram{\`e}r, Florian},
  booktitle    = {Advances in Neural Information Processing Systems
                  (NeurIPS) Datasets and Benchmarks Track},
  year         = {2024},
  url          = {https://arxiv.org/abs/2406.13352}
}

@inproceedings{zhan2024injecagent,
  title        = {{InjecAgent}: Benchmarking Indirect Prompt Injections in
                  Tool-Integrated Large Language Model Agents},
  author       = {Zhan, Qiusi and Liang, Zhixiang and Ying, Zifan and
                  Kang, Daniel},
  booktitle    = {Findings of the Association for Computational Linguistics:
                  {ACL} 2024},
  pages        = {10471--10506},
  year         = {2024},
  doi          = {10.18653/v1/2024.findings-acl.624},
  url          = {https://aclanthology.org/2024.findings-acl.624/}
}

@inproceedings{greshake2023notwhat,
  title        = {Not What You've Signed Up For: Compromising Real-World
                  {LLM}-Integrated Applications with Indirect Prompt
                  Injection},
  author       = {Greshake, Kai and Abdelnabi, Sahar and Mishra, Shailesh
                  and Endres, Christoph and Holz, Thorsten and
                  Fritz, Mario},
  booktitle    = {Proceedings of the 16th {ACM} Workshop on Artificial
                  Intelligence and Security ({AISec} '23)},
  pages        = {79--90},
  year         = {2023},
  doi          = {10.1145/3605764.3623985},
  url          = {https://doi.org/10.1145/3605764.3623985}
}

@article{zou2023universal,
  title        = {Universal and Transferable Adversarial Attacks on Aligned
                  Language Models},
  author       = {Zou, Andy and Wang, Zifan and Carlini, Nicholas and
                  Nasr, Milad and Kolter, J. Zico and Fredrikson, Matt},
  journal      = {arXiv preprint arXiv:2307.15043},
  year         = {2023},
  url          = {https://arxiv.org/abs/2307.15043}
}

@article{openai2024gpt4o,
  title        = {{GPT-4o} System Card},
  author       = {{OpenAI}},
  journal      = {arXiv preprint arXiv:2410.21276},
  year         = {2024},
  url          = {https://arxiv.org/abs/2410.21276}
}

@article{comanici2025gemini25,
  title        = {{Gemini 2.5}: Pushing the Frontier with Advanced Reasoning,
                  Multimodality, Long Context, and Next Generation Agentic
                  Capabilities},
  author       = {{Gemini Team}},
  journal      = {arXiv preprint arXiv:2507.06261},
  year         = {2025},
  url          = {https://arxiv.org/abs/2507.06261}
}

@article{grattafiori2024llama3,
  title        = {The {Llama 3} Herd of Models},
  author       = {{Llama Team}},
  journal      = {arXiv preprint arXiv:2407.21783},
  year         = {2024},
  url          = {https://arxiv.org/abs/2407.21783}
}

@article{qwen2024qwen25,
  title        = {{Qwen2.5} Technical Report},
  author       = {Yang, An and Yang, Baosong and Zhang, Beichen and Hui, Binyuan
                  and Zheng, Bo and Yu, Bowen and Li, Chengyuan and Liu,
                  Dayiheng and Huang, Fei and Wei, Haoran and Lin, Huan and
                  Yang, Jian and Tu, Jianhong and Zhang, Jianwei and Yang,
                  Jianxin and Yang, Jiaxi and Zhou, Jingren and Lin, Junyang
                  and Dang, Kai and Lu, Keming and Bao, Keqin and Yang, Kexin
                  and Yu, Le and Li, Mei and Xue, Mingfeng and Zhang, Pei and
                  Zhu, Qin and Men, Rui and Lin, Runji and Li, Tianhao and
                  Tang, Tianyi and Xia, Tingyu and Ren, Xingzhang and Ren,
                  Xuancheng and Fan, Yang and Su, Yang and Zhang, Yichang and
                  Wan, Yu and Liu, Yuqiong and Cui, Zeyu and Zhang, Zhenru and
                  Qiu, Zihan},
  journal      = {arXiv preprint arXiv:2412.15115},
  year         = {2024},
  url          = {https://arxiv.org/abs/2412.15115}
}

@article{denning1976lattice,
  title        = {A Lattice Model of Secure Information Flow},
  author       = {Denning, Dorothy E.},
  journal      = {Communications of the ACM},
  volume       = {19},
  number       = {5},
  pages        = {236-243},
  year         = {1976},
  doi          = {10.1145/360051.360056},
  url          = {https://doi.org/10.1145/360051.360056}
}

@article{volpano1996sound,
  title        = {A Sound Type System for Secure Flow Analysis},
  author       = {Volpano, Dennis and Irvine, Cynthia and Smith, Geoffrey},
  journal      = {Journal of Computer Security},
  volume       = {4},
  number       = {2--3},
  pages        = {167--187},
  year         = {1996},
  doi          = {10.3233/JCS-1996-42-304}
}

@inproceedings{smith2009qif,
  title        = {On the Foundations of Quantitative Information Flow},
  author       = {Smith, Geoffrey},
  booktitle    = {Foundations of Software Science and Computational Structures
                  (FoSSaCS)},
  series       = {Lecture Notes in Computer Science},
  volume       = {5504},
  pages        = {288--302},
  year         = {2009},
  publisher    = {Springer},
  doi          = {10.1007/978-3-642-00596-1_21}
}

@article{blackwell1953equivalent,
  title        = {Equivalent Comparisons of Experiments},
  author       = {Blackwell, David},
  journal      = {The Annals of Mathematical Statistics},
  volume       = {24},
  number       = {2},
  pages        = {265--272},
  year         = {1953},
  doi          = {10.1214/aoms/1177729032}
}

@inproceedings{podkopaev2021labelshift,
  title        = {Distribution-Free Uncertainty Quantification for
                  Classification Under Label Shift},
  author       = {Podkopaev, Aleksandr and Ramdas, Aaditya},
  booktitle    = {Proceedings of the Thirty-Seventh Conference on Uncertainty
                  in Artificial Intelligence},
  volume       = {161},
  pages        = {844--853},
  year         = {2021},
  publisher    = {PMLR},
  url          = {https://proceedings.mlr.press/v161/podkopaev21a.html}
}

@article{toolchaincrc2026,
  title        = {{ToolChain-CRC}: Conformal Risk Control for Agentic {AI}
                  Under Retrieval and Tool-Use Drift},
  author       = {Opoku, Jeffery and Banahene, David},
  journal      = {arXiv preprint arXiv:2606.18467},
  year         = {2026},
  url          = {https://arxiv.org/abs/2606.18467}
}

@article{hultberg2026anytime,
  title        = {Anytime-Valid Conformal Risk Control},
  author       = {Hultberg, Bj{\"o}rn and Zachariah, Dave and
                  Ribeiro, Ant{\^o}nio H.},
  journal      = {arXiv preprint arXiv:2602.04364},
  year         = {2026},
  url          = {https://arxiv.org/abs/2602.04364}
}

@article{khosravi2026selective,
  title        = {Conformal Selective Acting: Anytime-Valid Risk Control for
                  {RLVR}-Trained {LLMs}},
  author       = {Khosravi, Hassan and Huo, Xiaoming},
  journal      = {arXiv preprint arXiv:2605.20270},
  year         = {2026},
  url          = {https://arxiv.org/abs/2605.20270}
}
\end{document}